\documentclass[fontsize=12pt]{article}

\usepackage[english]{babel}

\usepackage[letterpaper,top=1in,bottom=1in,left=1in,right=1in]{geometry}

\usepackage{pdflscape}
\usepackage{adjustbox}
\usepackage{amsmath}
\usepackage{graphicx}
\usepackage[colorlinks=true, allcolors=blue]{hyperref}
\usepackage{natbib}
\usepackage{amssymb, amsthm}
\usepackage{booktabs}
\usepackage{multirow}
\usepackage{authblk}
\usepackage{lineno}
\usepackage{url}
\usepackage{enumerate}
\usepackage{comment}
\usepackage{changepage}
\usepackage{enumitem}
\usepackage{afterpage}

\usepackage{listings}
\usepackage{xcolor}

\definecolor{codegreen}{rgb}{0,0.6,0}
\definecolor{codegray}{rgb}{0.5,0.5,0.5}
\definecolor{codepurple}{rgb}{0.58,0,0.82}
\definecolor{backcolour}{rgb}{0.95,0.95,0.92}

\lstdefinestyle{mystyle}{
    commentstyle=\color{codegreen},
    keywordstyle=\bf,
    stringstyle=\color{magenta},
    basicstyle=\ttfamily,
    breakatwhitespace=false,
    breaklines=true,
    captionpos=b,
    keepspaces=true,
    numbers=none,
    numbersep=5pt,
    showspaces=false,
    showstringspaces=false,
    showtabs=false,
    tabsize=2
}

\usepackage{subcaption}
\DeclareCaptionFormat{custom}
{%
    \textbf{#1#2}\textit{#3}
}
\usepackage{xcolor}
\newcommand{\ob}[1]{{#1}}

\newtheorem{lemma}{Lemma}

\theoremstyle{definition}

\newtheorem{remark}{Remark}

\newtheoremstyle{break}
  {\topsep}{\topsep}%
  {\itshape}{}%
  {\bfseries}{.}%
  { }{}%
\theoremstyle{break}
\newtheorem{definition}{Definition}

\title{How to Evaluate Entity Resolution Systems:\\ An Entity-Centric Framework with Application to\\Inventor Name Disambiguation}

\author[1]{Olivier Binette}
\author[1]{Youngsoo Baek}
\author[2]{Siddharth Engineer}
\author[2]{Christina Jones}
\author[3]{Abel Dasylva}
\author[1]{Jerome P. Reiter}

\affil[1]{Duke University}
\affil[2]{American Institutes for Research}
\affil[3]{Statistics Canada}

\date{\today}

\begin{document}
\maketitle

\begin{abstract}
Entity resolution (record linkage, microclustering) systems are notoriously difficult to evaluate. Looking for a needle in a haystack, traditional evaluation methods use sophisticated, application-specific sampling schemes to find matching pairs of records among an immense number of non-matches. We propose an alternative that facilitates the creation of representative, reusable benchmark data sets without necessitating complex sampling schemes. These benchmark data sets can then be used for model training and a variety of evaluation tasks. Specifically, we propose an entity-centric data labeling methodology that integrates with a unified framework for monitoring summary statistics, estimating key performance metrics such as cluster and pairwise precision and recall, and analyzing root causes for errors. We validate the framework in an application to inventor name disambiguation and through simulation studies. Software: \url{https://github.com/OlivierBinette/er-evaluation/}
\end{abstract}

\section{Introduction}

Entity resolution is the process of identifying and linking database records referring to the same entity, such as a person or organization \citep{Christen2012, Christophides2019, Papadakis2021, Binette2022a}. In the absence of a reliable unique identifier, this is a large-scale clustering task: records need to be grouped into clusters, each representing a unique entity. In many applications, these clusters are numerous yet small. For example, entity resolution is used for the identification of unique inventors listed on U.S.\ Patents and Trademarks Office (USPTO) patents. This must account for commonalities in inventor names and the frequent occurrence of errors and variations in recorded names \citep{Li2014}. The task involves resolving millions of unique inventors, many of whom have authored only a handful of patents, while others may have contributed to hundreds.

Entity resolution systems typically employ machine learning and artificial intelligence models to tackle this challenge. The models use contextual information to predict whether or not two records refer to the same entity. For example, we can approximately resolve unique inventors by using patent topic, co-authors, employer, and location, in addition to first and last names \citep{Monath2021}. A large body of research has considered this problem, as this is an error-prone process that is further complicated by the scale of the clustering task \citep{Li2014, Huberty2014, Pezzoni2014, Balsmeier2015, Ventura2015, Kim2016, Yang2017, Doherr2017, han2019disambiguating, Yin2020, Monath2021, Doherr2021}.

In this article, we  address evaluating the accuracy of entity resolution systems. Traditional evaluation methods often rely on manually reviewing pairs of records to validate linkage predictions. However, finding matching pairs, especially those missed by the entity resolution system, is much like looking for a needle in a haystack: in a database of $n$ records, there are $\mathcal{O}(n^2)$ non-matches and only $\mathcal{O}(n)$ matches. Even if these matches can be found using a smart sampling scheme built around a specific entity resolution system, the resulting data are not necessarily well-suited for evaluating or training other models, or for estimating other metrics besides pairwise precision and recall.

Instead of reviewing pairs of records, our evaluation approach utilizes a sample of fully-resolved entities, {i.e., ground truth, or known clusters}. In this approach, all pairs within a resolved cluster are known to match, and any pair that intersects a resolved cluster but is not contained within it is known to be a non-match. For example, a resolved cluster of $10$ records in a database of $1,000,010$ records includes $ {10 \choose 2} = 45$ matching pairs and excludes $10$ million non-matching pairs. We show that sampling ground truth clusters, and using the resulting matches and non-matches, facilitates the estimation of performance metrics without having to rely on sophisticated pairwise sampling schemes that find sufficient numbers of matching pairs among the massive number of nonmatching pairs \cite[e.g., as in ][]{Marchant2017}.


In order to use fully-resolved entities as the starting point of evaluation, we propose an \textbf{entity-centric} evaluation framework with the following components.

\begin{description}[leftmargin=2em, labelindent=1em]
    \item[1. Cluster-Wise Error Metrics:] To identify errors made by an entity resolution system, we compare predicted clusters against a sample of known, fully-resolved clusters through error metrics defined at the record and cluster levels (section \ref{sec:error_space_definition}).
    \item[2. Global Performance Metric Estimates:] To obtain estimates of global performance metrics such as pairwise and b-cubed precision and recall \citep{Michelson2009, Menestrina2010, Barnes2015}, we express them as weighted aggregates of cluster-wise error metrics. This helps obtain estimates that are representative of the system’s performance on the entire data set, not just the benchmark (section \ref{sec:performance_estimation}).     
    \item[3. Error Analysis:]  To analyze the root causes of errors, we relate errors to entity features extracted from resolved clusters of records. (section \ref{sec:statistical-analyses}).
\end{description}
Furthermore, to support this process, we introduce:
\begin{description}[leftmargin=2em, labelindent=1em]
    \item[4. Data Labeling Though Cluster Sampling:]  A methodology for creating a benchmark set of fully-resolved entities through manual data labeling (section \ref{sec:data_labeling}).
    \item[5. Monitoring Statistics:] A set of summary statistics that serve to monitor the performance of entity resolution systems, even in the absence of a benchmark data set (section \ref{sec:summary_statistics}).
\end{description}
Our framework does \textit{black box} evaluation. That is, we evaluate the end result of an entity resolution system, without considering its specific architecture. This allows our framework to apply to any entity resolution system, as long as it produces a clustering as an output.

Figure \ref{fig:er-framework-diagram} represents the elements of the evaluation framework as well as their interdependencies.

\begin{figure}
    \centering
    \includegraphics{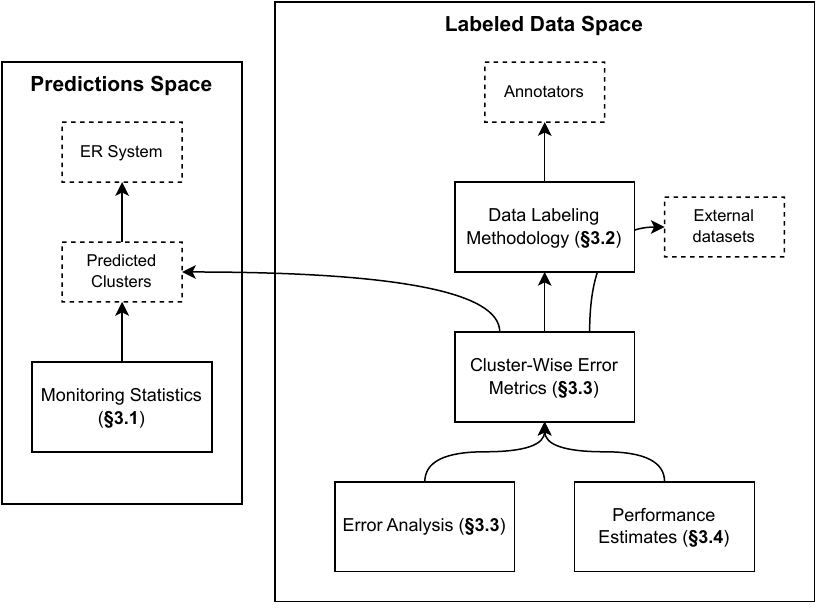}
    \caption{Diagram representation of the main elements of the framework and their dependencies. The entity resolution system and its predictions live in what we call the ``prediction space." Monitoring statistics can be computed for predictions, while their true value for a ground truth clustering can be estimated using labeled data. In the labeled data space, cluster-wise error metrics are obtained from external benchmark datasets or from our data labeling methodology. Error analysis and performance estimates rely on cluster-wise error metrics.}
    \label{fig:er-framework-diagram}
\end{figure}

\subsection{Previous Work}

Evaluation is a critical element of iterative model development, model selection, and validation of results. Despite the central importance of evaluation in the development and implementation of entity resolution systems, the topic has received scant attention in the literature. 

The most well-studied aspect of evaluation for entity resolution concerns the definition of performance evaluation metrics such as precision and recall \citep{Bilenko2003}, the b-cubed metric \citep{bagga1998algorithms}, generalized merge distances \citep{Maidasani2012}, and the use of crowdsourcing for data labeling \citep{Christophides2019}. However, this literature does not account for the statistical challenges involved in estimating these metrics from limited data (e.g., a non-representative benchmark data set or a small sample) which is necessary for their reliable use. The statistical literature focused on these estimation challenges is also limited \citep{Marchant2017, Dasylva2020, Binette2022b}. Furthermore, the topics of quality assurance, monitoring, and error analysis {appear mostly understudied in the entity resolution} literature.

\ob{Typical evaluation procedures used in entity resolution can result in misleading conclusions, as some commonly-used performance estimators are biased \citep{Wang2022, Binette2022b}. This can lead} to inaccurate representation of model performance and erroneous ranking of competing methods. In particular, the naive computation of pairwise precision on benchmark data sets has been shown to provide over-optimistic results \citep{Binette2022b}. Precision computed on benchmark data sets is often close to $1$, even when the true precision for the entire data set may be much lower. When combining biased precision estimates with recall estimates into an F1 score, this can lead to {performance rank reversals}: an algorithm may be assessed to perform better than another with high confidence, despite the opposite being true in practice \citep{Binette2022b}. \ob{In other words, comparing entity resolution algorithms based on F1 scores computed on benchmark data sets (such as in \cite{Yin2020}) leads to rankings not representative of performance on larger populations.}

\subsection{Outline of the Paper}

The rest of the paper is structured as follows. In section \ref{sec:background}, we provide background on our motivating application and current approaches to evaluation. In section \ref{sec:methods}, we introduce the proposed methodology.  In  section \ref{sec:results}, we showcase its application to inventor name disambiguation and its validation in simulation studies. Finally, we conclude in section \ref{sec:discussion} with a summary of our contributions and directions for future work.

\section{Background}\label{sec:background}

We begin this section with an overview of the disambiguation work carried out by the American Institutes for Research (AIR) for PatentsView.org.
We then discuss our motivating data and application in more detail. Finally, we provide additional information on current industry standards for the evaluation of entity resolution systems, using the methodology used by Statistics Canada as an example. 

\subsection{Patent Data Disambiguation for PatentsView.org}

PatentsView is a public patent data platform maintained by the American Institutes for Research and the U.S.\  Patents and Trademarks Office. It increases the value, utility, and transparency of U.S.\ patent data by providing enriched data products, data visualizations, and data exploration tools.

As one of its main contributions, PatentsView disambiguates patent inventors, assignees, lawyers, and locations. This embeds patent data in a large knowledge graph that links these individual entities through co-authorship, ownership, and citation relationships. However, this disambiguation is a significant challenge given the absence of unique identifiers for these entities, and given the large amount of noise and ambiguity in the data. PatentsView addresses this challenge by employing an assortment of disambiguation algorithms and updating the disambiguation for new data on a quarterly basis.

A variety and growing set of users have 
made use of PatentsView's data since its full launch in January 2017. Around 50,000 users visited the PatentsView website in 2023, with 75,000 downloads of bulk data files and an average of one million API requests per month. In Figure \ref{fig:patentsview}, we show the estimated numbers of citations to PatentsView in the academic literature over time and by Dewey Decimal subject classification.

\begin{figure}
    \centering
    \includegraphics[width=0.45\linewidth]{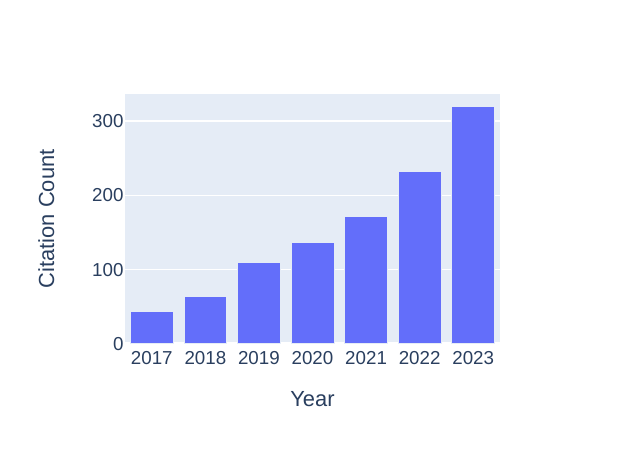}
    \includegraphics[width=0.45\linewidth]{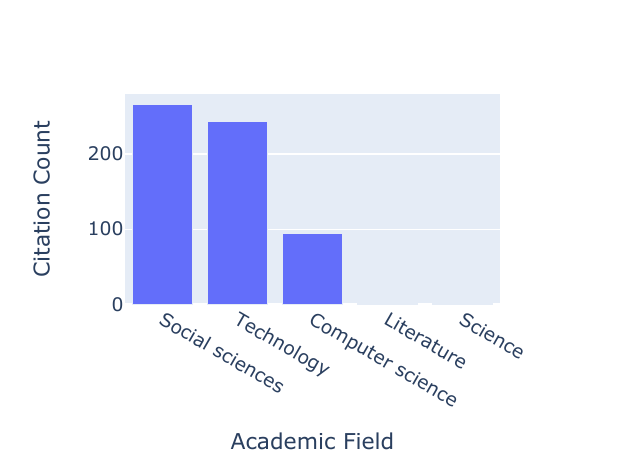}
    \caption{\textbf{Left:} Estimated number of citations to PatentsView in academic literature by year. \textbf{Right:} Number of citations by estimated Dewey Decimal Broad Classification. The estimated citation numbers were obtained by searching Google Scholar for mentions to ``PatentsView'' and ``Patents View'' and reviewing all results, with the 2023 year estimate containing extrapolated counts for November and December. The Dewey Decimal Classification categories were obtained by extracting abstracts from papers and programmatically querying openAI's GPT-3.5 model for a classification estimate. Note that GPT-3.5 could not ascertain the classification code for 189 papers.}
    \label{fig:patentsview}
\end{figure}

\subsection{Motivating Data and Application}\label{sec:motivating_data}

We consider inventor data from U.S.\ patents granted between 1976 and November 2023, inclusively, obtained from the bulk data download page of PatentsView.org \citep{uspto2023datadownload}. \ob{These data are} aggregated in a single table indexed by {inventor mentions}, where each row of this table is referred to as a {record}. An inventor mention is a reference to a specific inventor (identified by authorship sequence number) on a specific patent (identified by patent number). For each inventor mention, we have information such as inventor first and last name, inventor location, patent co-inventor names, patent title, patent abstract, patent application date, patent grant date, patent assignee, and patent classification \ob{codes; see Table \ref{tab:data-example} for examples.}

Additionally, we have the results of PatentsView's disambiguation algorithm applied to 18 individual data releases between August 2017 and November 2023. At each release, more inventor mentions were disambiguated as new patents were granted. Furthermore, the disambiguation algorithm was occasionally tweaked in this time period. For each data release, a disambiguation is available that consists of a membership vector assigning each inventor mention to a unique inventor identifier. Our evaluation framework specifically targets this sequence of disambiguation results.

\begin{table}[t]
\caption{\label{tab:data-example} Example of attributes available for individual inventor mentions. Additional attributes, such as co-inventor information and patent text, are available as well.}
\centering
\begin{adjustbox}{width=\textwidth}
\begin{tabular}{lllllllllll}
\toprule
\textbf{Mention ID} & \textbf{Patent} & \textbf{Name} & \textbf{City} & \textbf{State} & \textbf{Country} & \textbf{Year} & \textbf{Title} & \textbf{Kind} & \textbf{Assignees} & \textbf{Classification} \\
\midrule
US8501339-3 & 8501339 & Lutgard C. De Jonghe & Lafayette & CA & US & 2013 & Protected lit... & B2 & ['PolyPlus Battery Company'] & H01G \\
US6085742-1 & 6085742 & Stuart Lindsay & Tempe & AZ & US & 2000 & Intrapulmonary del... & A & ['Aeromax Technologies, Inc.'] & A61M \\
US9215286-6 & 9215286 & Elin R. Pedersen & Portola Valley & CA & US & 2015 & Creating a social & B1 & ['Goolge Inc'] & H04L \\
\bottomrule
\end{tabular}
\end{adjustbox}
\end{table}

A number of benchmark data sets are available to help assess the performance of PatentsView's disambiguation system \citep{Binette2022c}. Most of these data sets provide ``ground truth" disambiguation (of various quality) of a specific subset of inventors for a given time period. They are all cluster samples that fit our evaluation framework. However, except for the new inventor benchmark data developed in \citep{Binette2022b, Binette2022c}, none of the benchmark data sets are representative of the full population of inventors or have sampling weights associated with them.

\subsection{Industry Standards for Entity Resolution Evaluation}\label{sec:industry_standards}

We look at Statistics Canada to understand common industry practices for the evaluation of entity resolution systems. The institution has been at the forefront of record linkage since the seminal paper of \cite{fellegi_theory_1969}, who were then working at Statistics Canada, and with Fellegi becoming Chief Statistian of Canada from 1985 to 2008. Research groups at Statistics Canada have continued to advance the field in recent years, including on the topic of evaluation.

For context, Statistics Canada routinely performs the linkage of many data sets including censuses, administrative data, and survey data. These linkages are classified according to their purpose as analytical or operational. In the former case,  the main goal is to produce a linked file including all the required responses and explanatory variables, to fit some regression model. A good example is the linkage of the Canadian Community Health Survey to the Canadian mortality database \citep{sanmartin_et_al_2016}. Many of the analytical linkages are implemented within the Social Data Linkage Environment \citep{sdle_webpage_2022}. All other linkages are classified as operational, such as linkages that support specific steps in a sample survey, e.g., frame maintenance operations, like removing duplicates or linking different sampling frames in a multi-frame survey. For all the linkages, regardless of their purpose, the accuracy of the linkage decisions is measured with clerical reviews or a statistical model.

When the purpose is analytical, a linked data set is created that may end up in a research data center, including quality indicators for the data users.
According to \cite{qian_et_al_2021}, the indicators should comprise the linkage rate, the linkage representativeness (LR) according to \cite{vanderlaan_baaker_2015}, and the precision and the false negative rate
to characterize the linkage accuracy at the level of the record pairs. The LR metric highlights potential biases, here linkage rate differences across subgroups, by measuring the variance of estimated linkage propensities.

At Statistics Canada, accuracy measures are estimated primarily through clerical reviews \citep{dasylva_et_al_review_2016} that can be carried out with G-LINK, the agency's generalized system for probabilistic record linkage. These reviews consist of visual inspections on a probability sample of record pairs to determine if their constituent records refer to the same unit. When the linkage is probabilistic, the pairs are typically stratified according to their linkage weights \citep[chap. 5]{millard_blanchard_2022}. For quality control, the same pair may be reviewed by two or more clerks \citep{dasylva_et_al_review_2016}.
 For analysis, further activities are conducted to evaluate the linkage quality \citep[chap. 6, chap. 7.5]{rlppm_2017}.

For operational linkages, the linkage accuracy may also be measured with clerical reviews as in the census over-coverage study \citep[chap. 8.2.1]{statcan_2016_census_coverage}, where the reviews are based on sampling groups of connected records \citep{dasylva_et_al_2015}. These groups are essentially connected components in the graph where the vertices are the records and the edges are the links. To cut costs, Statistics Canada is also considering model-based estimates 
\citep{dasylva_goussanou_jjsds_2022}. For example, this approach was used to set the record similarity threshold (i.e., how similar two records have to be before they are linked) in the probabilistic linkage between the 2021 census of agriculture and the census of population for the same year \citep{statcan_2021_ag_census_linkage}.

\section{Methodology}\label{sec:methods}

We now describe our evaluation framework in detail.

Let $\mathcal{R}$ be a set of $N$ records. Each record in $\mathcal{R}$ is a reference to a unique entity (an entity mention), with some entities being referred to by multiple records. Two records are said to be coreferent, or to match if they refer to the same entity. Note that coreference is an equivalence relation: if records $A$ and $B$ refer to the same entity, and records $B$ and $C$ refer to the same entity, then records $A$ and $C$ also refer to the same entity. As such, coreference induces a clustering $\mathcal{C}$ of $\mathcal{R}$, with two records being in the same cluster if they refer to the same entity. We refer to $\mathcal{C}$ as the ``true" clustering, i.e., it represents the true identity relations that entity resolution aims to recover. An entity resolution system outputs a predicted clustering $\hat{\mathcal{C}}$. To evaluate the entity resolution system, we introduce a methodology that samples clusters $c\in \mathcal{C}$ and use them to assess the accuracy of $\hat{\mathcal{C}}$. Throughout, we use the $c(r)$ to denote the cluster in $\mathcal{C}$ containing a given record $r \in \mathcal{R}$. Similarly, we let $\hat c(r)$ be the predicted cluster in $\hat{\mathcal{C}}$ containing a given $r \in \mathcal{R}$. 

To fix ideas, we can take $\mathcal{R}$ as the set of inventor mentions on granted U.S. patents. The clustering $\mathcal{C}$ represents the true grouping of inventor mentions according to the real-world inventor that they represent, i.e., each cluster $c\in\mathcal{C}$ links to a set of patents authored by this inventor. The predicted clustering $\hat{\mathcal{C}}$ is the output of an inventor disambiguation system.

\subsection{Summary Statistics and Quality Assurance}\label{sec:summary_statistics}

The first component of our evaluation framework is a set of summary statistics \ob{that} 
describe properties of any given disambiguation result. The goal of these statistics 
is to provide key indicators \ob{that} can be tracked to understand and monitor disambiguation results throughout the lifetime of an entity resolution system. They are simple and easily interpretable statistics \ob{that} can help explain properties of the clustering. Additionally, these statistics act as quality assurance indicators \ob{that} can be automatically monitored to identify potential bugs and errors. In particular, estimated summary statistics can be monitored using standard quality assurance tools such as control charts \citep{montgomery2020introduction} and anomaly detection \citep{chandola2009anomaly}.


The statistics we propose are described below.
\begin{definition}[{\bf Cluster Size Distribution Statistics}] \label{def1}
    Let $\mathcal{C}$ be a clustering of a set of records $\mathcal{R}$. We define the following metrics regarding the distribution of cluster sizes in $\mathcal{C}$.
    \begin{description}[leftmargin=2em, labelindent=1em]
        \item[Average Cluster Size:] The average cluster size is defined as $R_{\text{size}} = \frac{1}{\lvert \mathcal{C} \rvert}\sum_{c\in \mathcal{C}} \lvert c \rvert$.
        \item[Matching Rate:] The 
        matching rate statistic $R_m$ is the proportion of records $r \in \mathcal{R}$ that are linked to some other records, {i.e., are part of a cluster with at least two records.} Namely, $R_m = \frac{1}{\lvert \mathcal{R} \rvert} \sum_{r\in \mathcal{R}}\mathbb{I}(\lvert c(r) \rvert > 1)$.
        \item[Cluster Hill Numbers (Entropy Curve):] 
        For a given order $q \geq 0$, the corresponding Hill number $H_q$ is the exponentiation of the Rényi entropy of order $q$ of the cluster size distribution. That is, given a clustering $\mathcal{C}$, we have
        \begin{equation}\label{eq:hill_numbers}
            H_q = \left(\sum_{i=1}^N \left(\mathbb{P}(\lvert c\rvert = i)\right)^q \right)^{1/(1-q)},
        \end{equation}
       where $\mathbb{P}(\lvert c\rvert = i) = \tfrac{1}{\lvert\mathcal{C}\rvert}\sum_{c \in \mathcal{C}} \mathbb{I}(|c|=i)$. Here, \eqref{eq:hill_numbers} is 
        continuously extended \ob{at $q=0$ and $q=1$, and taking the limit from the left for $q=\infty$.}
        \end{description}
\end{definition}
The set of Hill numbers $\{H_q\}_{q\geq 0}$ uniquely characterizes the ordered cluster size distribution and provides interpretable statistics. For instance, $H_0$ is the number of unique cluster sizes in the distribution; $H_1$ is the exponential Shannon entropy; $H_2$ is the inverse of the probability that two random inventors have authored the same number of patents; and, $H_\infty$ is the prevalence of the most common cluster size.
Since Hill numbers are continuous in their parameter $q$, they provide a simple representation of the ordered cluster size distribution as a continuous curve.

The next set of metrics quantifies the level of noise in the data. Suppose that each cluster element $r$ is associated with a label, such as an inventor's name listed on a patent. For a given inventor, listed names may differ on different patent applications. Additionally, multiple inventors may share the same name. To quantify these two situations, we introduce the {homonymy rate} and {name variation rate} statistics.

\begin{definition}[Variation and Homonymy Rate Statistics]\label{def2}
    Let $\mathcal{C}$ be a clustering of a set of records $\mathcal{R}$, where each record is associated with a label $s_r$, 
    and let $n(r)$ be the set of records with the same labels as $r$, i.e., $n(r) = \{r'\in \mathcal{R} \mid s_{r'} = s_r\}$. We define the following metrics to describe label similarity across clusters and label variation within clusters.
    \begin{description}[leftmargin=2em, labelindent=1em]
        \item[Homonymy Rate:] The homonymy rate $R_h$ is the proportion of clusters containing a record that shares its label with another cluster. That is, the homonymy rate is defined as
        \begin{equation}\label{eq:homonymy_rate}
            R_h = \frac{1}{\lvert \mathcal{C} \rvert} \sum_{c \in \mathcal{C}} \mathbb{I}(\exists r \in c \,:\, n(r) \not\subset c(r)).
        \end{equation}
        \item[Name Variation Rate:] The name variation rate $R_v$ is the proportion of clusters with variation among the record labels. That is, the name variation rate is defined as
        \begin{equation}\label{eq:name_variation_rate}
            R_v = \frac{1}{\lvert \mathcal{C} \rvert} \sum_{c \in \mathcal{C}} \mathbb{I}(\exists r \in c \,:\, c(r) \not\subset n(r)).
        \end{equation}
    \end{description}
\end{definition}

\begin{remark}
Various measures of cluster homogeneity/compactness and separation have been proposed in the literature and are used as part of the objective functions of unsupervised clustering algorithms \citep{Johnson1967, Davies1979, Liu2010, duran2013cluster}. When using clustering algorithms relying on such statistics, these can be reported in addition to our proposed summary statistics.
\end{remark}



The above summary statistics can either be applied to the predicted clustering $\hat{\mathcal{C}}$ (replacing $\mathcal{C}$ by $\hat{\mathcal{C}}$ in definitions \ref{def1} and \ref{def2}), or estimated for the unknown ground truth clustering $\mathcal{C}$ by using a sample of clusters. That is, suppose that a sample of $k$ clusters, $c_1,\ldots,c_k$,
 is taken from the true clustering $\mathcal{C}$. We write $p_c$ for the probability (up to a normalizing constant) that a given cluster $c$ is sampled in any one draw. For example, with probability proportional to size sampling, we have $p_c = \lvert c \rvert / N$. Sampling processes and designs are discussed in more detail in section \ref{sec:data_labeling}.

To facilitate estimation of the {average cluster size} and {matching rate} in Definition \ref{def1}, we re-express 
  $R_{size}$ and   $R_{m}$ 
 as ratios of expectations. In particular, suppose we randomly sample a single cluster $c$, where the probability of sampling a given $c \in \mathcal{C}$ is proportional to $p_c >0$. The quantities in the numerators and denominators of $R_{size}$ and $R_{m}$ can be rewritten as   expectations with respect to this sampling distribution. 
That is, 
\begin{equation}\label{eq:ratio_rep}
    R_{size} = \frac{\mathbb{E}[|c|/p_c]}{\mathbb{E}[1/p_c]}, \quad R_{m} = \frac{\mathbb{E}[|c|\mathbb{I}(|c|>1)/p_c]}{\mathbb{E}[|c|/p_c]}.
\end{equation}
For example, when sampling clusters with probability proportional to size, we have $p_c = |c|/N$. Consequently, $\mathbb{E}[1/p_c] = \sum_{c \in \mathcal{C}} [1/p_c] p_c = |\mathcal{C}|$, which matches the denominator of $R_{size}$ in Definition \ref{def1}.   
We then use the sample of $k$ clusters to estimate the expectations in the numerators and denominators separately---for example, we can estimate each expectation with its corresponding weighted sample average---and  take the ratio of the relevant  estimated expectations. 
Alternatively, as described in section \ref{sec:estimators}, we also can use the bias-adjusted estimator and variance estimator described in \cite{Binette2022b}.
Similarly, the \textbf{homonymy rate} and \textbf{name variation rate} can be written as ratios of expectations corresponding to the numerators and denominators of \eqref{eq:homonymy_rate} and \eqref{eq:name_variation_rate}, which then can be estimated  in the same way.

\ob{Estimating the {Hill numbers} 
is a more challenging 
task. Several proposals for estimating a whole curve of Hill numbers exist in ecological and biological diversity estimation \citep{Chao1984,Chao2013,Chao2014}. These apply in our context only when considering uniform sampling weights. We therefore leave the problem of estimating Hill numbers for future work.}

\subsection{Data Labeling Methodology}\label{sec:data_labeling}

Data labeling is often needed to construct benchmark data sets suitable for evaluation. Here, we \ob{expand and formalize} the methodology introduced in \cite{Binette2022b} for practical and cost-effective data labeling.

Fundamentally, our goal is to obtain a probability sample 
of known, ground truth clusters. To do so, we suggest sampling $k$ records $r_1, \dots, r_k$ 
and having data annotators recover the associated clusters $c(r_1), \dots, c(r_k)$. 

A variety of tools can help labelers build these ground truth clusters, including using predicted clusters or blocking as a starting point, and using search tools for identifying candidate matches. Below, we suggest a simple methodology that we have found useful in our application.

The methodology relies on two main components, namely 
(a) a predicted disambiguation used as the starting point of the data labeling, and (b) a search tool used to identify candidate matches for manual review.
For (a), we use an entity resolution algorithm to provide a set of predicted clusters to aid in the disambiguation. In the application to PatentsView.org described in \cite{Binette2022b}, 
the current inventor disambiguation was used as the starting point of data labeling. Otherwise, a simple exact name-matching disambiguation could be used. For (b), we used PatentsView.org's search tool. It is also possible to use spreadsheets to parse through subsets of records, or to use elasticsearch \citep{elasticsearch} as a search backend. The use of elasticsearch provides tolerance to typographical errors and intuitive term-based search functionality that may be familiar to data labelers.

Given a population of records $\mathcal{R}$ and a sample size $k$, the data labeling methodology is as follows:
\begin{enumerate}[label=(\roman*)]
    \item First, sample a sequence of records $S = (r_1, r_2, \dots, r_k)$,  $r_i \in \mathcal{R}$.
    \item For each sampled record $r$, we recover the corresponding {predicted} cluster $\hat c(r)$ from (a) above.  We then perform  the following three steps:
        \begin{enumerate}[label=(\Alph*)]
            \item The data labeler takes note of \textbf{overclustering errors}: records in the predicted cluster $\hat c(r)$ that are not part of $c(r)$. We denote by $A_r$ the set of overclustering errors that the data labeler aims to identify, defined as
            \begin{equation}\label{eq:def_A_r}
                A_r = \hat c(r) \backslash c(r).
    \end{equation}
            \item The data labeler takes note of \textbf{underclustering errors}: records that are in $c(r)$ but that are not in $\hat c(r)$. These can be found by using the search tool (per \ob{(b)} above).  We denote by $B_r$ this set of underclustering errors that the data labeler aims to identify, defined as
            \begin{equation}\label{eq:def_B_r}
                B_r = c(r) \backslash \hat c(r).
    \end{equation}
            \item Given $A_r$ and $B_r$, the true cluster associated with record $r$ is $c(r) = \hat c(r) \backslash A_r \cup B_r$.  
        \end{enumerate}
\end{enumerate}
At the end of this process, we have a sequence of sampled ground truth {clusters $c(r_1), c(r_2),  \dots, c(r_k)$ associated with each sampled record. Put together, the ground truth clusters form a benchmark data set $C_S$ that {can} be used for evaluation.}

\subsubsection{Sampling Schemes} \label{sec:probability_sampling}



Many different designs can be used to sample clusters. We recommend randomly sampling records with replacement and finding associated clusters, which is straightforward to implement and facilitates estimation.
In this case, the probability that a given cluster $c$ is sampled in any single draw is $\lvert c \rvert / N$, i.e., this is sampling clusters with probability proportional to their sizes.


Sampling clusters with probability proportional to $|c|$ can result in increased accuracy relative to simple random sampling of clusters.  In particular, when a cluster-level outcome of interest is correlated with the size of the cluster, the probability proportional to size design offers smaller standard errors in estimates of population quantities \citep{lohr2021sampling}.  This is the case for PatentsView data, as large clusters tend to be associated with increased chances for errors.



Other sampling designs can be leveraged, in which case the probabilities $p_c$ should change to match the design. For example, uniform sampling probabilities $p_c \propto 1$ can be appropriate when entities, rather than records, are sampled at random. As another example, suppose the disambiguation algorithm provides match probabilities between all pairs of records;  that is, we have probabilities $p_{r,r'}$ that records $r, r' \in \mathcal{R}$ are a match for all $(r, r')$. Then, for a given record $r \in \mathcal{R}$, the expected number of records to be removed to the predicted cluster $\hat c(r)$ in step (A) of the data labeling methodology, the expected overclustering error from \eqref{eq:def_A_r}, is
\begin{equation}
    \mathbb{E}[\lvert A_r \rvert ] = \sum_{r' \in \hat c(r)} (1-p_{r, r'}).
\end{equation}
Similarly, the expected number of records to be added to the predicted cluster $\hat c(r)$ in step (B) of the data labeling methodology, the expected underclustering error from \eqref{eq:def_B_r}, is
\begin{equation}
    \mathbb{E}[\lvert B_r \rvert ] = \sum_{r' \in \mathcal{R}\backslash \hat c(r)} p_{r, r'}.
\end{equation}
We can sample 
records with probabilities proportional to the sum of these two expectations, which could facilitate more accurate estimation of key metrics.  However, computing $p_c$ would be more complicated in this design, as it requires fitting a probabilistic record linkage model as a first step.

\subsubsection{Quality Control}

Following data labeling, a quality control step is used to identify obvious errors in the labeling. Through automated methods and validation with the data labeler, this helps correct typographical errors and unintentional errors without changing the intent of the labeler.

The first two properties used for quality control are the fact that $A_r \subset \hat c(r)$ and that $r \not \in A_r$. If the set of overclustering errors identified by a data labeler does not satisfy these constraints, then an error was made. To help identify errors in the set of underclustering errors identified by a data labeler, one can look for records that are not part of the same block as $r$, or that have highly dissimilar attributes to $r$. These simple validations can catch most typographical and annotation errors in our experience.

\subsection{Error Analysis}\label{sec:error_analysis}

We now consider the problem of analyzing entity resolution errors identified through data labeling. This is done in complement to performance estimation (section \ref{sec:performance_estimation}) that provides representative performance metrics, such as precision and recall, for a given population of records $\mathcal{R}$.

To motivate error analysis, note that adequate performance as measured by performance metrics is a necessary but insufficient characteristic of machine learning systems \citep{Zhang2022}. Complex or black-box machine learning systems can fail in intricate or unexpected ways that are not acceptable, even when good overall performance is achieved \citep{Oakden-Rayner2020}. For example, systematic failure to disambiguate inventor names from a given culture, for instance due to differences in naming conventions, could be considered an unacceptable flaw even if the relative prevalence of such inventors is low. It is therefore necessary to test systems and to investigate errors to help identify such issues \citep{Zhang2022, Poth2020}. Broadly, the goals of error analysis are to:
\begin{enumerate}[label=(\roman*)]
    \item Identify patterns in the error space. This includes identifying areas of low performance and performance disparities between subgroups.
    \item Identify systematic failures and their cause. For example, a simple systematic failure may be related to the use of punctuation marks in names.
\end{enumerate}

Our approach to error analysis has two main steps.
In section \ref{sec:error_space_definition}, we define record-wise and cluster-wise error metrics to obtain an interpretable and relevant error space to analyze. 
In section \ref{sec:statistical-analyses}, we show how to analyze performance disparities by subgroups and we perform error auditing to identify common causes for errors.
\ob{Note that many other} approaches from the machine learning testing literature could also be relevant \ob{but go beyond the scope of this paper \citep{Murphy2008, Ramanathan2016, Zhang2022, Braiek2020, Aggarwal2019, Tuncali2020}.}

\subsubsection{Error Metrics Defined at the Record and Cluster Levels}\label{sec:error_space_definition}

We propose metrics to quantify the errors made by a predicted clustering $\hat{\mathcal{C}}$, 
defined as follows.
\begin{definition}[Record-Wise Error Metrics]
    Let $\mathcal{C}$ be a clustering of a set of records $\mathcal{R}$, let $\hat{\mathcal{C}}$ be a predicted clustering of $\mathcal{R}$, and let $r \in \mathcal{R}$ be a given record. We define the following error metrics for comparing the true entity cluster $c(r)$ associated with $r$ to the predicted cluster $\hat c(r) \in \hat{\mathcal{C}}$ associated with $r$:
    \begin{description}[leftmargin=2em, labelindent=1em]
        \item[Error Indicator ($\texttt{EI}$): ] This is a binary error indicator defined as $\texttt{EI}(r) = 0$ when the predicted cluster $\hat c(r)$ equals the true cluster $ c(r)$, and $\texttt{EI}(r) = 1$ otherwise.
        \item[Size Difference Error ($\texttt{SDE}$):] This is the difference in size between the predicted cluster $\hat c(r)$ and the true cluster $c(r)$, defined as
        $
            \texttt{SDE}(r) = \lvert \hat c(r) \rvert - \lvert c(r) \rvert.
        $
        \item[Overclustering Error ($\texttt{OCE}$):] This is the number of records in the predicted cluster $\hat c(r)$ that are not part of the true cluster $c(r)$, defined as 
        $
            \texttt{OCE}(r) = \lvert A_r \rvert = \lvert \hat c(r) \backslash c(r) \rvert.
        $
        \item[Underclustering Error ($\texttt{UCE}(r)$):] This is the number of records in the predicted cluster $\hat c(r)$ that are not part of the true cluster $c(r)$, defined as  
        $
            \texttt{UCE}(r) = \lvert B_r \rvert = \lvert c(r) \backslash \hat c(r) \rvert.
        $
    \end{description}
    Additionally, we define the \textbf{relative overclustering error} ($\texttt{ROCE}$) as $\texttt{ROCE}(r) = \texttt{OCE}(r)/\lvert \hat c(r) \rvert$ and the \textbf{relative underclustering error} ($\texttt{RUCE}$) as $\texttt{RUCE}(r) = \texttt{UCE}(r)/ \lvert c(r) \rvert$.
\end{definition}

For a given cluster  $c \in \mathcal{C}$, we can define corresponding metrics by  averaging over its internal records $r \in c$.
\begin{definition}[Cluster-Wise Error Metrics]
    Given a cluster $c$ and a record-wise error metric $\texttt{E}$, we extend $\texttt{E}$ to clusters by defining $\texttt{E}(c) = \frac{1}{\lvert c \rvert} \sum_{r \in c} \texttt{E}(r)$, the average of record-wise error metrics within the cluster. Specifically, we define the cluster-wise error metrics
    $$
        \texttt{OCE}(c) = \frac{1}{\lvert c \rvert} \sum_{r \in c} \texttt{OCE}(r), \quad \texttt{UCE}(c) = \frac{1}{\lvert c \rvert} \sum_{r \in c} \texttt{UCE}(r), \quad \texttt{SDE}(c) = \frac{1}{\lvert c \rvert} \sum_{r \in c} \texttt{SDE}(r),
    $$
    and
    $$
        \texttt{ROCE}(c) = \frac{1}{\lvert c \rvert} \sum_{r \in c} \texttt{ROCE}(r), \quad \texttt{RUCE}(c) = \frac{1}{\lvert c \rvert} \sum_{r \in c} \texttt{RUCE}(r), \quad \texttt{EI}(c) = \frac{1}{\lvert c \rvert} \sum_{r \in c} \texttt{EI}(r).
    $$
\end{definition}


\subsubsection{Error Auditing and Statistical Analyses}\label{sec:statistical-analyses}

To understand errors and their causes, we consider cluster-wise error metrics, looking at the characteristics of individual errors, and analyzing their relationship with features of interest. We choose to focus on errors at the cluster level rather than at the record level since this can be more interpretable in some applications, including for inventor disambiguation tasks. Here, the clusters represent invidivual inventors, and errors for disambiguating a given inventor can be understood in view of inventor characteristics.

To illustrate, for PatentsView's inventor disambiguation, we consider two analyses based on cluster-wise error metrics.
    First, we consider marginal performance disparity between imputed {inventors'} ethnicities by computing performance metric estimators (see section \ref{sec:performance_estimation}) within subgroups. We visualize the results by adapting the performance bias module of the  Deepchecks Python package \citep{Chorev_Deepchecks_A_Library} to include uncertainty quantification. Here, ethnicity is imputed using the Ethnicolr Python package \citep{Laohaprapanon_ethnicolr_Predict_Race_2022} with a model trained on the 2010 Census Surname Files. These types of models have inaccuracies,  
    but they can help uncover failure modes that would otherwise remain hidden \citep{Jain2022Importance}.

    Second, to audit and classify errors, we manually review the list of errors, identify meaningful categories, and report error rates across the categories (see Figure \ref{fig:error-auditing}). The manual review is assisted by a Streamlit web app \citep{streamlit} to browse disambiguated clusters, visualize the differences between predicted and ground truth clusters, and log error tags (see Figure \ref{fig:streamlit}). Specifically, for a given disambiguated inventor, the user is shown a scatterplot of inventor mentions organized by membership to predicted clusters on the vertical axis versus true clusters on the horizontal axis. All inventor mentions associated with predicted clusters that intersect the true cluster are shown, ensuring that both overclustering and underclustering errors can be visualized. Hovering over an inventor mention's data point shows related information, including stated name, assignee, location, patent title, patent grant date, and co-author last names. Additionally, a raw data table can be explored, sorted, and searched, in order to analyze errors in more details. For each disambiguated inventor in a review sample, overclustering and underclustering errors are tagged if present according to a potential cause for error.

\begin{figure}[ht]
    \centering
    \includegraphics[height=5in]{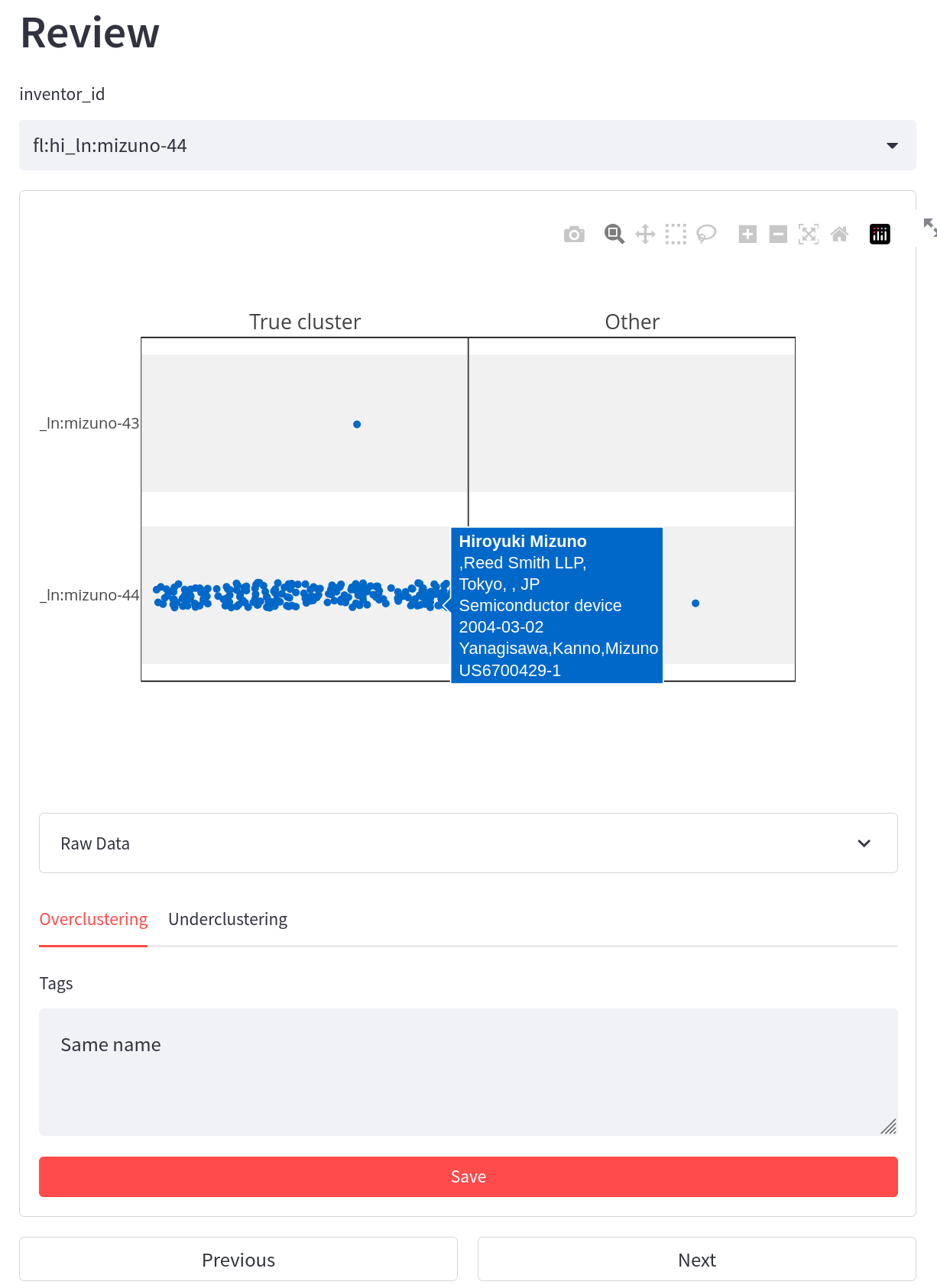}
    \caption{Screenshot of the Streamlit app used for clerical error review. The ``inventor\_id'' field at the top of the page selects a ground truth cluster whose label is derived from the predicted cluster used as a starting point. The table below shows how this ground truth cluster (first horizontal column) relates to predicted clusters it intersects with on the x axis, with each point representing an inventor mention. Observations regarding overclustering and underclustering errors are recorded below.}
    \label{fig:streamlit}
\end{figure}

\begin{remark}
    Additional error analysis techniques can be relevant in some applications. For instance, decision tree classification models can be used to relate cluster error metrics with cluster features, as done in the SliceFinder algorithm \citep{Chung2019}. The ER-Evaluation Python package \citep{Binette2023joss} implements decision tree fitting and visualization tools for this purpose.
\end{remark}
\begin{remark}
Defining performance metrics for subgroups requires some care. Suppose that a subset of records $\mathcal{R}' \subset \mathcal{R}$ is associated with a group of entities of interest, such as inventors of a given ethnicity. Naively, it may be tempting to restrict both $\hat{\mathcal{C}}$ and $\mathcal{C}$ to $\mathcal{R}'$, before computing cluster-wise error metrics and estimators. This is not the correct approach, as this is blind to errors in predicted links between records in $\mathcal{R}'$ and records not in $\mathcal{R}'$, and will artificially inflate performance metrics estimates. The correct approach is to first identify entities $\mathcal{C'}\subset \mathcal{C}$ corresponding to the subgroup of interest. Then, cluster-wise error metrics can be computed for sampled clusters that fall within $\mathcal{C'}$, and performance estimates can be obtained by restricting the sample to this subset. For example, with pairwise precision, this corresponds to estimating the ratio of the number of pairwise links within clusters in $\mathcal{C}'$ (true links) to the number of links in $\hat{\mathcal{C}}$ that intersect $\mathcal{C}'$ (predicted links involving $\mathcal{C}'$).
\end{remark}

\subsection{Performance Metric Estimation}\label{sec:performance_estimation}

We now turn to the problem of estimating performance evaluation metrics based on benchmark (i.e., labeled) data sets. We assume that the benchmark data sets take the form of a probability sample of true clusters $C_S = (c_1, \dots, c_k)$. We denote by $p_c, c\in \mathcal{C}$, the per-instance sampling probabilities, up to a global normalizing constant.

Many performance metrics commonly used for entity resolution (pairwise, b-cubed, and cluster metrics) can be expressed in terms of the overclustering ($\texttt{OCE})$ and underclustering ($\texttt{UCE}$) error metrics defined in \eqref{eq:def_A_r}, \eqref{eq:def_B_r}, and section \ref{sec:error_analysis}, together with functions of the predicted clustering $\hat{\mathcal{C}}$. This representation has the advantage of directly relating the data labeling process from section \ref{sec:data_labeling} to performance evaluation metrics. As such, labeling uncertainty can be propagated to the estimation of performance evaluation metrics. Furthermore, this representation disaggregates performance evaluation metrics in terms of cluster-level performance, allowing fine-grained error analysis as shown in section \ref{sec:error_analysis}. Finally, this representation provides a unified framework for performance estimation in terms of cluster-wise error rates, allowing for efficient computation of all metrics and estimators from a single table containing overclustering and underclustering error metrics. The framework can be extended to the estimation of additional metrics through the use or specification of appropriate record-level error metrics, as shown in section \ref{sec:extension_example}.

\subsubsection{Representation Lemmas}\label{sec:representation-lemmas}

We now provide the expressions for performance metrics that we use to derive estimators.

\paragraph{Pairwise Precision and Recall}

Let $\mathcal{P}$ be the set of pairs of elements belonging to the same cluster in $\widehat{\mathcal{C}}$ (predicted pairs) and let $\mathcal{T}$ be the set of pairs of elements belonging to the same cluster in $\mathcal{C}$ (true pairs). Precision $P$ and recall $R$ are defined as
\begin{equation}
    P = \frac{\lvert \mathcal{T} \cap \mathcal{P} \rvert}{\lvert \mathcal{P} \rvert}, \quad R = \frac{\lvert \mathcal{T} \cap \mathcal{P} \rvert}{\lvert \mathcal{T} \rvert}.
\end{equation}

Lemma \ref{lemma:pairwise_precision_recall} expresses precision and recall as ratios 
involving the error metrics defined in section \ref{sec:error_space_definition}.   

\begin{lemma}\label{lemma:pairwise_precision_recall}
    Suppose we sample one cluster $c$ from $\mathcal{C}$ at random. Let $p_c >0$ be proportional to its sampling probability. Then
    \begin{equation}
        P = \frac{\mathbb{E}\left[ \lvert c \rvert ( \lvert c \rvert -1 - \texttt{UCE}(c)) / p_c \right]}{\mathbb{E}\left[\lvert c \rvert (\lvert c \rvert - 1 + \texttt{SDE}(c)) / p_c\right]}, \quad R = \frac{\mathbb{E}\left[ \lvert c \rvert ( \lvert c \rvert -1 - \texttt{UCE}(c)) / p_c \right]}{\mathbb{E}\left[\lvert c \rvert (\lvert c \rvert - 1) / p_c\right]}.  
    \end{equation}
\end{lemma}

\paragraph{Pairwise F-Score}

Let $F_\beta$, $\beta > 0$, be the weighted harmonic mean between precision and recall, namely
\begin{equation}
    F_\beta = \left(\frac{P^{-1} + \beta^2 R^{-1}}{1+\beta^{2}}\right)^{-1}.
\end{equation}

Lemma \ref{lemma:f-score} provides an expression for $F_\beta$ in terms of our error metrics.

\begin{lemma}\label{lemma:f-score}
     Suppose we sample one cluster $c$ from $\mathcal{C}$ at random. Let $p_c >0$ be proportional to its sampling probability. Then
    \begin{equation}
        F_\beta = \frac{\mathbb{E}\left[\lvert c \rvert \left( \lvert c \rvert -1 - \texttt{UCE}(c)\right) / p_c \right]}{\mathbb{E}\left[\lvert c \rvert \left(\lvert c \rvert - 1 + \tfrac{1}{1+\beta^2}\texttt{SDE}(c)\right) / p_c\right]}.
    \end{equation}
\end{lemma}

\paragraph{Cluster Precision and Recall}

Following \cite{Menestrina2010}, we define cluster precision $cP$, cluster recall $cR$, and cluster $F$-score as
\begin{equation}
    cP = \frac{\lvert \mathcal{C} \cap \hat{\mathcal{C}} \rvert}{ \lvert \hat{\mathcal{C}} \rvert},\quad cR = \frac{\lvert \mathcal{C} \cap \hat{\mathcal{C}} \rvert}{ \lvert {\mathcal{C}} \rvert}, \quad cF_\beta = \left(\frac{cP^{-1} + \beta^2 cR^{-1}}{1+\beta^{2}}\right)^{-1}.
\end{equation}
That is, $cP$ is the proportion of correctly predicted clusters among all predicted clusters, and $cR$ is the proportion of correctly predicted clusters among all true clusters.

\begin{lemma}\label{lemma:cluster_metrics}
           Suppose we sample one cluster $c$ from $\mathcal{C}$ at random. Let $p_c >0$ be proportional to its sampling probability. Then
    \begin{equation}
        cP = \frac{N \mathbb{E}\left[ \texttt{EI}(c)/ p_c \right]}{\lvert \hat{\mathcal{C}} \rvert \mathbb{E}\left[ \lvert c \rvert /p_c \right]}, \quad cR = \frac{\mathbb{E}\left[\texttt{EI}(c) / p_c \right]}{\mathbb{E}\left[ 1 / p_c \right]}, \quad cF_\beta = \frac{\mathbb{E}\left[ N(1+\beta^2) \texttt{EI}(c)/ p_c \right]}{\mathbb{E}\left[ (N \beta^2  + \lvert \hat{\mathcal{C}} \rvert\lvert c \rvert)/ p_c \right]}.
    \end{equation}
\end{lemma}

\paragraph{B-Cubed Precision and Recall}

The b-cubed (or $B^3$) precision and recall \citep{bagga1998algorithms}, here placing equal weight on each ground truth cluster, are defined as
\begin{equation}
    P_{{B^3}} = \frac{1}{\lvert \mathcal{C}\rvert}\sum_{c \in \mathcal{C}} \frac{1}{\lvert c \rvert} \sum_{r \in c} \frac{\lvert c(r) \cap \hat c(r)\rvert }{\lvert \hat c(r) \rvert}, \quad R_{{B^3}} = \frac{1}{\lvert \mathcal{C}\rvert}\sum_{c \in \mathcal{C}} \frac{1}{\lvert c \rvert} \sum_{r \in c} \frac{\lvert c(r) \cap \hat c(r)\rvert }{\lvert c(r) \rvert}.
\end{equation}

Lemma \ref{lemma:b_cubed_metrics} expresses the b-cubed metrics in terms of the relative expected numbers of missing or extraneous links.

\begin{lemma}\label{lemma:b_cubed_metrics}
        Suppose we sample one cluster $c$ from $\mathcal{C}$ at random. Let $p_c >0$ be proportional to its sampling probability. Then
    \begin{equation}
        P_{{B^3}} =\frac{\mathbb{E}\left[ (1-\texttt{ROCE}(c)) / p_c \right]}{\mathbb{E}\left[ 1 / p_c \right]}, \quad R_{{B^3}} = \frac{\mathbb{E}\left[ (1-\texttt{RUCE}(c)) / p_c \right]}{\mathbb{E}\left[ 1 / p_c \right]}.
    \end{equation}
\end{lemma}

\subsubsection{Performance Estimators}\label{sec:estimators}

All of the expressions in section \ref{sec:representation-lemmas} are of the form, 
\begin{equation}\label{eq:theta}
    \theta = {\mathbb{E}[f(c)]} \big / {\mathbb{E}[g(c)]},
\end{equation}
for two functions $f$ and $g$, and where the expectations are taken with respect to sampling the random cluster $c$. Under the assumption that $c_1,  \dots, c_k$ are sampled with replacement, we can estimate $\mathbb{E}[f(c)]$  and $\mathbb{E}[g(c)]$ 
using the empirical averages $\bar{f}_k$ and $\bar{g}_k$, respectively, where 
\begin{align}\label{eq:estimatefg}
\bar{f}_k = \frac{1}{k} \sum_{i=1}^{k} f(c_i), \;\;\; 
\bar{g}_k = \frac{1}{k} \sum_{i=1}^{k} g(c_i).
\end{align}
We take the ratio of these averages, and thus obtain a ratio etimator for $\theta$. We make further adjustments to reduce bias in the ratio estimator and its variance estimator using the approach described in \cite{Binette2022b}. That is, our estimate of quantities of the form \eqref{eq:theta}, given samples clusters $c_1, \dots, c_k$, is
\begin{align}\label{eq:bias_adjustment}
    \hat{\theta} = \frac{\bar{f}_k}{\bar{g}_k} \left\{ 1 + \frac{1}{k(k-1)} \sum_{i=1}^k \frac{g(c_i)}{\bar{g}_k} \left(\frac{f(c_i)}{\bar{f}_k} - \frac{g(c_i)}{\bar{g}_k} \right) \right\}.
\end{align}
Our variance estimate is
\begin{align} \label{eq:estimated_variance}
    \widehat{V}(\hat{\theta}) = \left( \frac{\bar{f}_k}{\bar{g}_k} \right)^2 \frac{1}{k(k-1)} \sum_{i=1}^k \left( \frac{g(c_i)}{\bar{g}_k} - \frac{f(c_i)}{\bar{f}_k}\right)^2.
\end{align}

\subsubsection{Example Extension to an Additional Metric}\label{sec:extension_example}

Our framework can be extended to estimate additional metrics. For example, consider the cluster homogeneity metric, defined as the normalized conditional entropy between the true and predicted clusterings \citep{rosenberg2007v}. That is, homogeneity is defined as
\begin{equation}\label{eq:lemma_5_1}
    h = 1- \frac{H(\mathcal{C} \mid \hat{\mathcal{C}})}{H(\mathcal{C})},
\end{equation}
where
\begin{equation}\label{eq:lemma_5_2}
    H(\mathcal{C} \mid \hat{\mathcal{C}}) = - \sum_{c\in \mathcal{C}} \sum_{\hat c \in \hat{\mathcal{C}}} \frac{\lvert c \cap \hat c \rvert}{N} \log \frac{\lvert c \cap \hat c \rvert}{\lvert \hat c \rvert}\quad \text{and}\quad H(\mathcal{C}) = - \sum_{c \in \mathcal{C}} \frac{\lvert c \rvert}{N} \log \frac{\lvert c \rvert}{N}.
\end{equation}
Define the record-wise error metric $\texttt{H}$ as
\begin{equation}\label{eq:lemma_5_3}
    \texttt{H}(r) = (\lvert \hat c(r) \rvert - \texttt{OCE}(r)) \log \frac{\lvert \hat c(r) \rvert - \texttt{OCE}(r)}{\lvert \hat c(r) \rvert}
\end{equation}
and, for a cluster $c\in \mathcal{C}$, define the cluster-wise variant $\texttt{H}(c) = \frac{1}{\lvert c \rvert} \sum_{r \in c} \texttt{H}(r)$.
\begin{lemma}\label{lemma:lemma_5}
    Suppose we sample one cluster $c$ from $\mathcal{C}$ at random with probability proportional to positive numbers $p_c > 0$, $c \in \mathcal{C}$. Then
    \begin{equation}\label{eq:homogeneity_lemma}
        h = 1 - \frac{\mathbb{E}\left[ \lvert c \rvert \texttt{H}(c) / p_c\right]}{\mathbb{E}\left[ \lvert c \rvert \log(\lvert c \rvert / N) /p_c\right]}.
    \end{equation}
\end{lemma}

Similarly as before, expression \eqref{eq:homogeneity_lemma} can be used to define a ratio estimator of cluster homogeneity.

\section{Empirical Illustrations and Simulations}\label{sec:results}

In this section, we first showcase the application our evaluation framework to PatentsView's inventor disambiguations. We then present results of a simulation study to assess the accuracy of our performance metric estimators. 

Our data labeling was performed using the methodology described in section \ref{sec:data_labeling}, resulting in a set of 400 cluster samples representative of data up to December 31, 2022. This benchmark data set is a direct extension of the work of \cite{Binette2022b}, where some practical details of the data labeling process are explained in more detail.

\subsection{Summary Statistics and Quality Assurance}

Figure \ref{fig:summary} displays our summary statistics computed using    PatentsView's predicted inventor disambiguations $\hat{\mathcal{C}}$ as a function of time. For Hill numbers, we focus on $H_0$, the number of distinct cluster sizes, and $H_1$, the exponentiated Shannon entropy.  Figure \ref{fig:summary} also displays our 
 estimates of the summary statistics for $\mathcal{C}$, excluding $H_0$ and $H_1$, for disambiguations carried out on or before December 31, 2021 (black dotted line).\footnote{We cannot estimate the true value of summary statistics values for later disambiguations since our benchmark data set only covers  records up to December 31, 2021.}

\begin{figure}[ht]
    \centering
    \includegraphics[width=\linewidth]{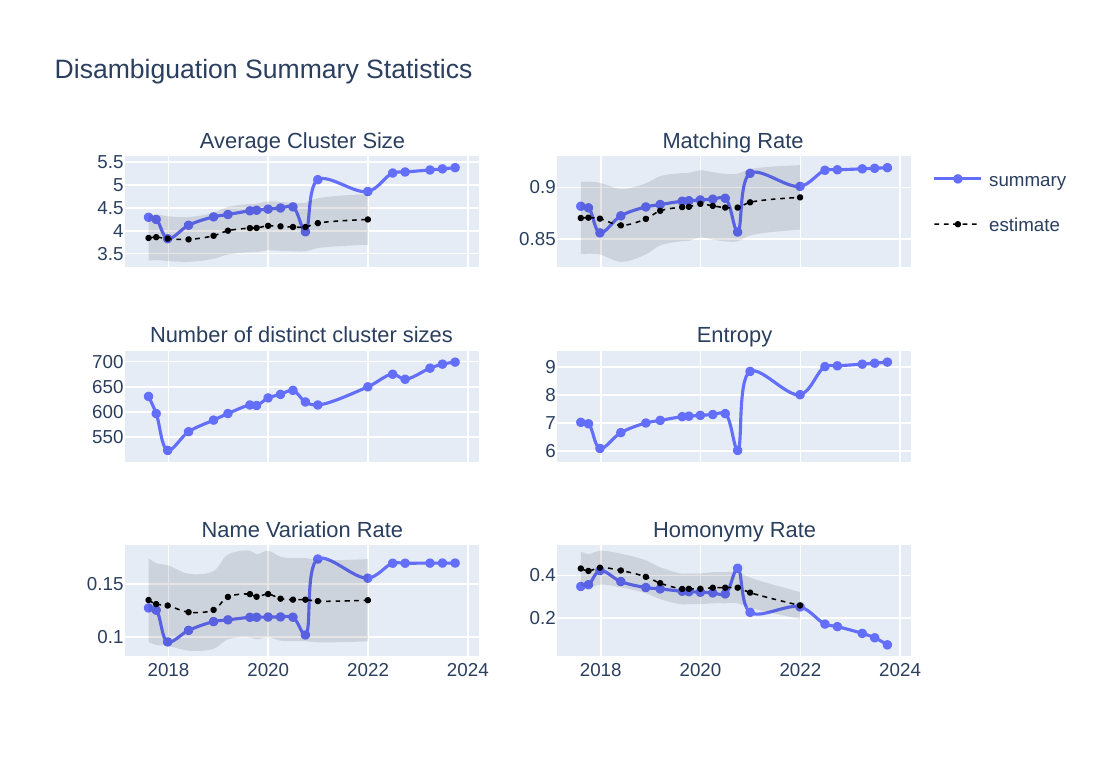}
    \caption{\textbf{Blue line:} Summary statistics for PatentsView's history of predicted disambiguations. \textbf{Black dotted line:} Estimates of the true value of the summary statistics, based on the 2022 inventors' benchmark data, with pointwise $95\%$ confidence intervals.}
    \label{fig:summary}
\end{figure}

Figure \ref{fig:summary} reveals several features of the evolution of these summary statistics and their estimates. First, the average cluster size, matching rate, {and name variation rate statistics computed with $\hat{\mathcal{C}}$ jump quite significantly around 2021, suggesting something unusual has happened in that time frame.} Second, the number of distinct cluster sizes is roughly monotonic, except for early disambiguation history and again around 2021. The monotonic trend is in line with expectations, as the number of distinct cluster sizes should increase as data are added over time; {the sudden break in the trend around 2021 is not.}  The homonymy rate statistic drops over time, going down to nearly 5\% by 2024, meaning that almost 95\% of inventor's names are assumed to be unique in the predicted disambiguation.
{The noticeable changes around 2021 coincide with a change to the disambiguation algorithm, apparently one that impacts the properties of  the clusterings.} We observe that summary statistics from $\hat{\mathcal{C}}$ are mostly within the 
confidence intervals for the corresponding quantities in $\mathcal{C}$. This suggests that the disambiguation algorithm generates clusterings with similar properties (as measured by these summary statistics) as the true clustering. Regardless, the rather significant changes in 2021 {should} motivate further investigation to ensure that the data still meet quality expectations.

One challenge with interpreting Figure \ref{fig:summary} is that both the data and algorithm change over time. It is possible to separate these two aspects by considering the evolution of summary statistics for a fixed subset of the data. In Figure \ref{fig:summary_restricted}, we consider inventor mentions from before August 2017 as a fixed data set {over time}. This is the largest data subset that was disambiguated at all available time points, allowing us to see the evolution of summary statistics over PatentsView's entire history.


The change patterns observed in Figure \ref{fig:summary} are accentuated in Figure \ref{fig:summary_restricted}.
For inventor mentions dating from before August 2017, the average cluster size and the homonymy rate from $\hat{\mathcal{C}}$ now fall outside of the $95\%$ confidence intervals,
even though that was quite not the case in Figure \ref{fig:summary}. 
Evidently, the quality of the pre-2017 data disambiguation has been affected by changes to the algorithm over the period. 
In fact, these changes were made to account for the significant amount of new data incorporated in the years between 2017 and 2021.  This observation highlights the importance of considering the effect of change both in algorithms and in amount of data when assessing disambiguation quality. Indeed, as we have previously noted, the difficulty of entity resolution problems is not constant across data sizes; rather, the opportunity for errors grows quadratically as a function of data size. Changes made to account for growing data, and specifically to rebalance precision and recall, since false match errors increase the fastest, will necessarily affect the characteristics of the disambiguation of data subsets. 

\begin{figure}[ht]
    \centering
    \includegraphics[width=\linewidth]{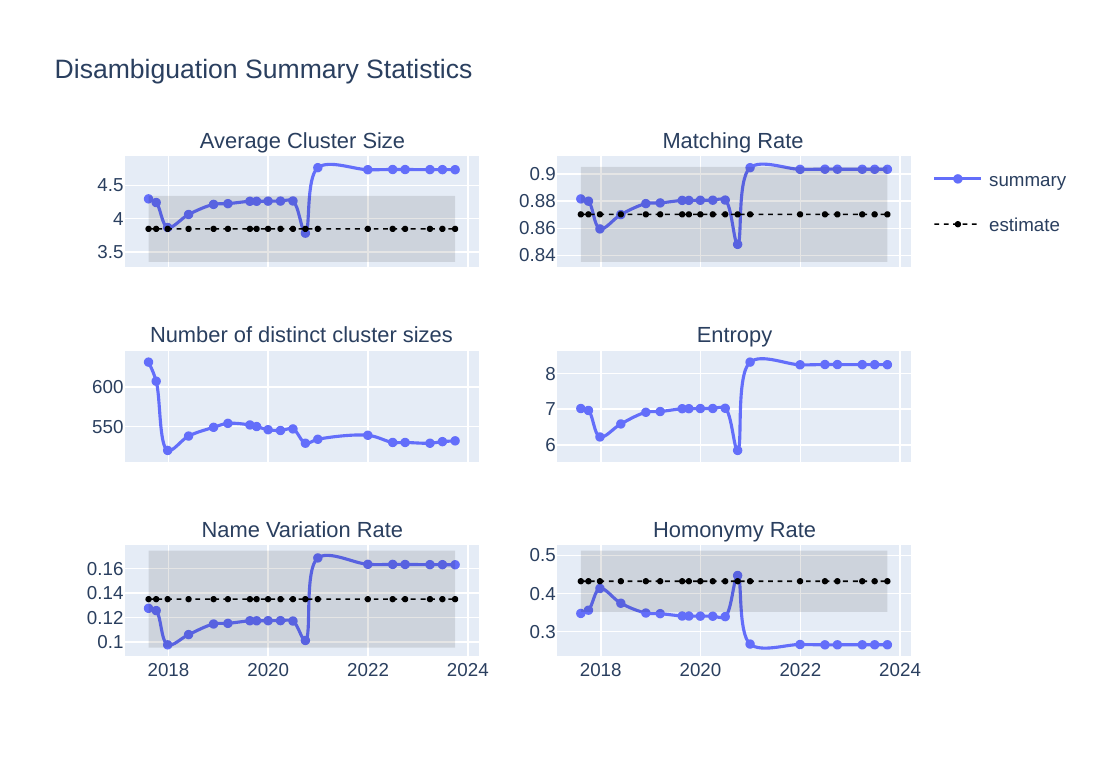}
    \caption{Summary statistics and estimates for the fixed data set of inventor mentions dating up to August 2017. Disambiguations of this fixed data set have changed over time, as changes to the algorithm were made and since information from additional records was used to resolve entities. As before, the dotted line is the estimate of the summary value for the true clustering of the August 2017 inventor mentions. The shaded bands are pointwise $95\%$ confidence intervals. Since the data set is fixed in this case, the estimates are constant over time.}
    \label{fig:summary_restricted}
\end{figure}

Overall, we recommend using summary statistics as a monitoring and quality control tool, tracking both global behavior and properties of fixed data subsets. Unexpected behaviors, sudden changes, or incompatibility with representative estimates, should trigger an investigation to validate the quality of the system's {inputs and outputs}.  

\subsection{Performance Estimates}\label{sec:results_estimates}

Figure \ref{fig:estimates} displays performance estimates over PatentsView's disambiguation history, with plus or minus one standard deviation confidence intervals. There is an important dip in performance before the beginning of 2021, which was then corrected. Performance in view of these metric estimates has been mostly stable since 2022, which provides some assurance in the quality of the linkages despite the changes in the algorithm. 

\begin{figure}[ht]
    \centering
    \includegraphics[width=\linewidth]{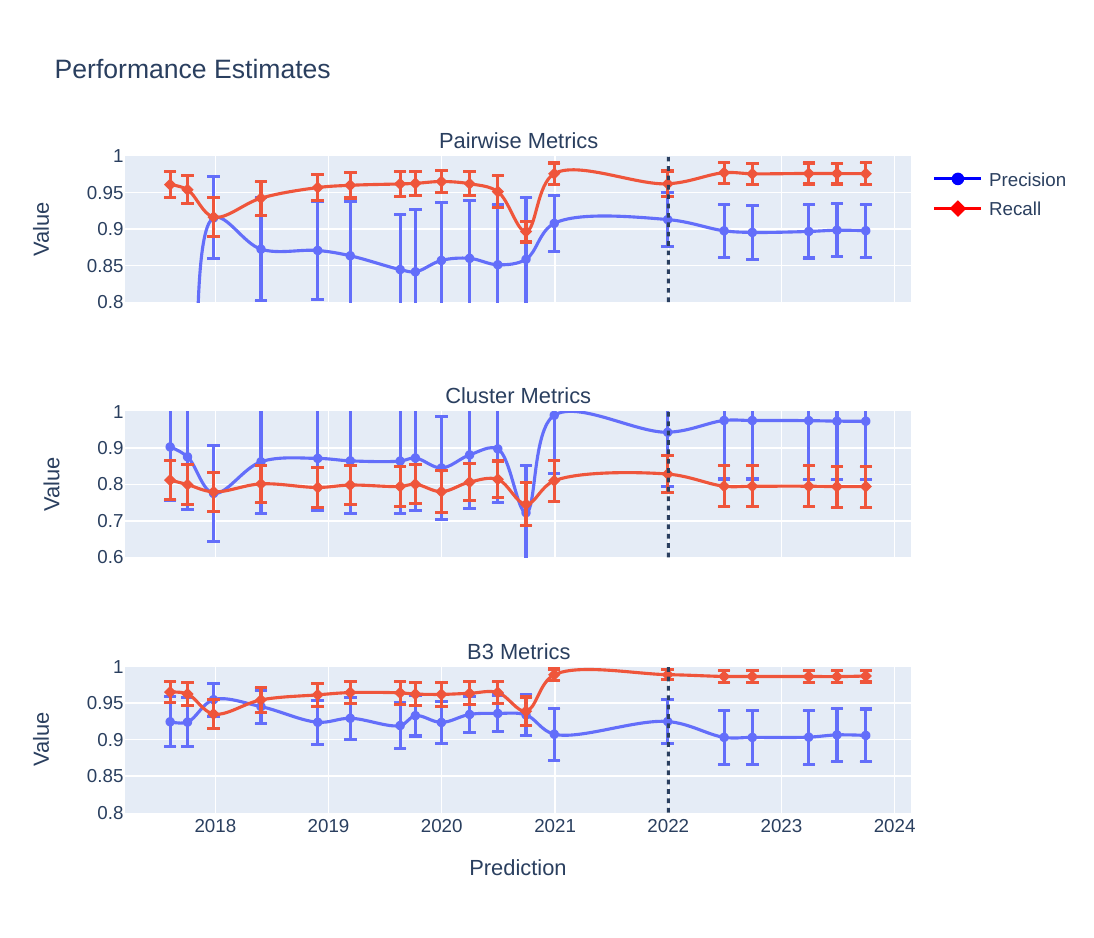}
    \caption{Performance metrics estimates and confidence intervals (plus or minus one estimated standard deviation) over PatentsView's disambiguation history. The ground truth data only cover inventor mentions up to December 31, 2021. As such, post-2021 estimates correspond to the accuracy of the disambiguation for data up to that point.}
    \label{fig:estimates}
\end{figure}

The estimated pairwise precision generally has remained lower than the estimated pairwise recall, and yet estimated cluster precision has been generally higher than estimated cluster recall. There are two things to note about this. First, uncertainty for the cluster precision estimates is 
large. This is because our sampling scheme, sampling clusters with probability proportional to size, leads to relatively 
few sampled instances of small clusters, despite small clusters being highly prevalent in the data. With cluster metrics putting equal weight on all clusters, this leads to higher estimated standard deviation. It would be possible to reduce uncertainty for the cluster metric estimates by {increasing the sample size or by } specifying an alternate sampling scheme that is 
more likely to result in the observation of small clusters, {for example, by stratifying based on the predicted cluster size corresponding to each record before sampling with probability proportional to size}. This could be relatively inexpensive, as manually reviewing small clusters is typically
faster than reviewing large clusters.

Furthermore, it is surprising to see pairwise recall estimates being higher than pairwise precision estimates since PatentsView has aimed to provide higher precision than recall (three of our authors have been directly involved in PatentsView). One key application of these metric estimates is to better align the accuracy and characteristics of entity resolution with business objectives. Accurate performance metric estimates can be used as objective functions for training machine learning models, for performing model selection, or for calibrating a given model. Accuracy objectives and the relative balance between metrics can be accounted for to satisfy requirements.

\subsection{Error analysis}

We now turn to the analysis of errors, their causes, and their relationship with features of interest.  Figure \ref{fig:error-auditing} displays the weighted relative frequency of observations made by a clerical reviewer (the first author) when analyzing errors presented using our error auditing tool. In practice, we would want to perform error auditing in two passes. The first would be a brainstorming session, taking notes to identify and define meaningful categories and labels that can be applied to different kinds of errors. A second pass then would implement the strategy derived from the first stage, helping provide actionable insights into causes of errors. Common issues could be investigated further by a development team. 

\begin{figure}[ht]
    \centering
    \includegraphics[width=\linewidth]{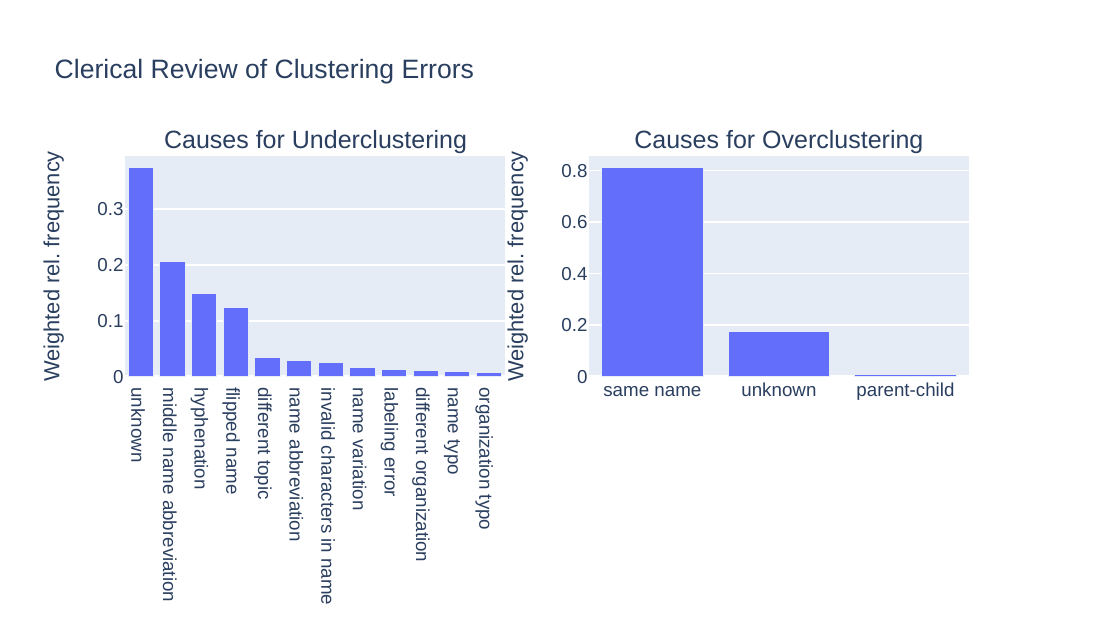}
    \caption{Reviewer's notes and their weighted relative frequencies for patterns in overclustering and underclustering errors.}
    \label{fig:error-auditing}
\end{figure}

Here, we show the raw data from the first brainstorming step, as it helps illustrate what kinds of observations and insights were first made. One common label is ``unknown.'' This label is applied when an entity is not correctly disambiguated, but it was not clear what could have been the cause of the error. As examples, a case could be particularly ambiguous; there could be an error in the data labeling; or, there could be insufficient contextual information in the data to justify combining two predicted clusters or separating one predicted cluster into two clusters. A second common label is ``same name,'' assigned to overclustering errors. This represents cases where inventor mentions were merged because of a shared name, even though other contextual information pointed towards the two representing different inventors. Otherwise, common underclustering errors are associated with variation in name spelling, such as a middle name being abbreviated or not, a name being hyphenated or not, a first and last name being written in one order or the other (which is common in certain cultures), or, less frequently, a typographical error in a name or an invalid character. A few underclustering errors are associated with the dissimilarity of patent topics or a typographical error in the spelling of the assigned organization. For a few cases, 
a labeling error might be a cause for the error. Note that the error auditing did not aim at finding errors in the data labeling, and so this label only represents anecdotal observations.

As previously noted, following a first observation and brainstorming step, a precise error auditing plan should be prepared. This plan should include clear definitions of a specific set of labels that can be applied to certain error cases. Following the application of this second step, results regarding key issues in the disambiguation can be communicated and used to consider mitigation methods.

\begin{figure}[ht]
    \centering
    \includegraphics[width=\linewidth]{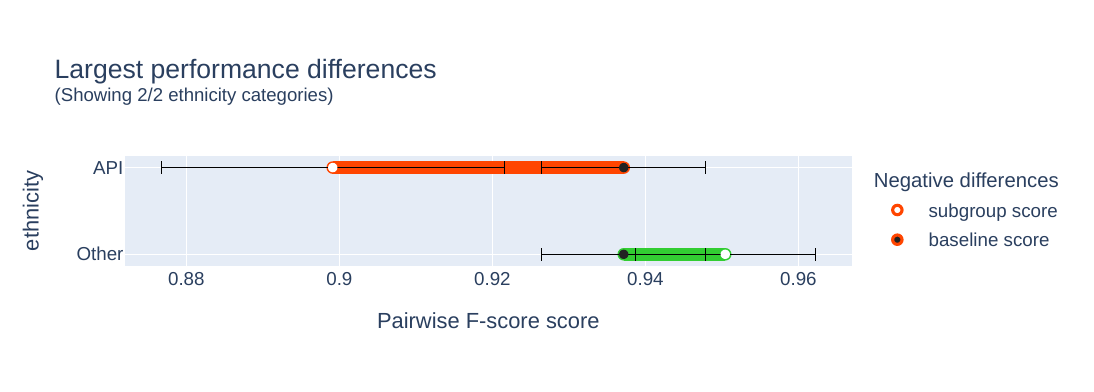}
    \caption{Performance difference from the baseline for inventors with an inferred Asian and Pacific Islander ethnicity (API) or other inferred ethnicity.}
    \label{fig:performance-bias}
\end{figure}

Figure \ref{fig:performance-bias} displays performance disparities between inventors with an inferred Asian and Pacific Islander ethnicity (API) or other inferred ethnicity. There is an estimated $5\%$ difference in F-score across these two subgroups. Although this is not a very large difference considering the uncertainty in estimation, analyzing performance disparities in this way can help identify subgroups for which more attention should be directed in the design of disambiguation algorithms.

\subsection{Simulation Study}\label{sec:simulations}

To conclude this section, we present two simulation studies. The first is based on the RLData10000 data set, and the second 
second uses PatentsView data up to May 28, 2018. 

The RLData10000 simulation is designed to evaluate the effectiveness of sampling with probabilities proportional to cluster size when clusters are small. 
We use a predicted disambiguation with high pairwise precision ($91\%$) and high pairwise recall ($97\%$), as we believe this makes for a challenging estimation task. Indeed, many clusters will be correctly disambiguated in this case, and therefore few errors will be observed. 

The PatentsView simulation is designed to validate the accuracy of our estimators specifically when applied to PatentsView's data. We use the December 30, 2021,  disambiguation as a ``ground truth,'' as we believe it is a close approximation to the true clustering in terms of cluster size distribution. We use the May 28, 2018 disambiguation as a prediction of the ``ground truth", as it is one of the earliest reliable disambiguations produced by PatentsView.
The May 28, 2018, disambiguation has $91\%$ pairwise precision and $94\%$ pairwise recall when compared to the December 30, 2021, disambiguation. This is a high accuracy bar, which makes the performance estimation problem challenging.

\subsubsection{RLData10000 Simulation}\label{sec:rldata-simulation}


\begin{figure}[ht]
    \centering
    \includegraphics[width=0.85\linewidth]{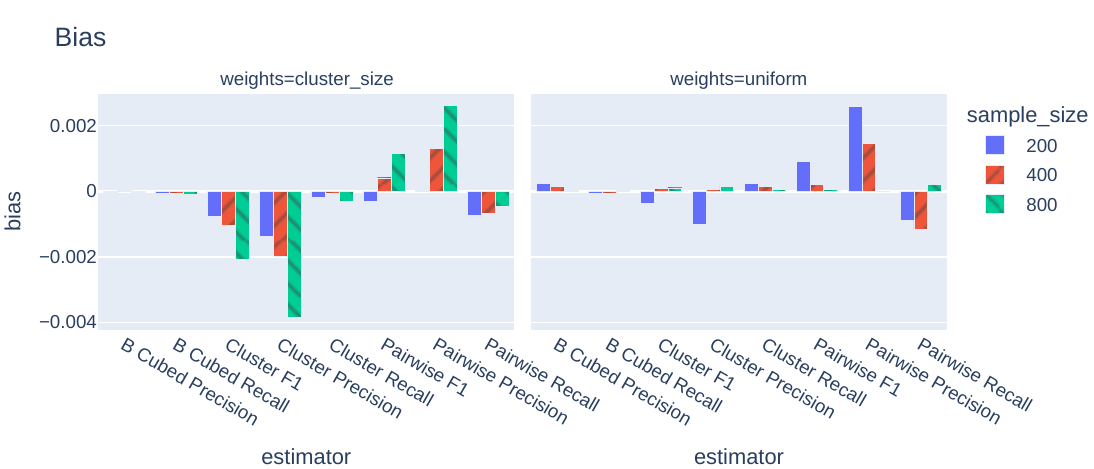}
    \includegraphics[width=0.85\linewidth]{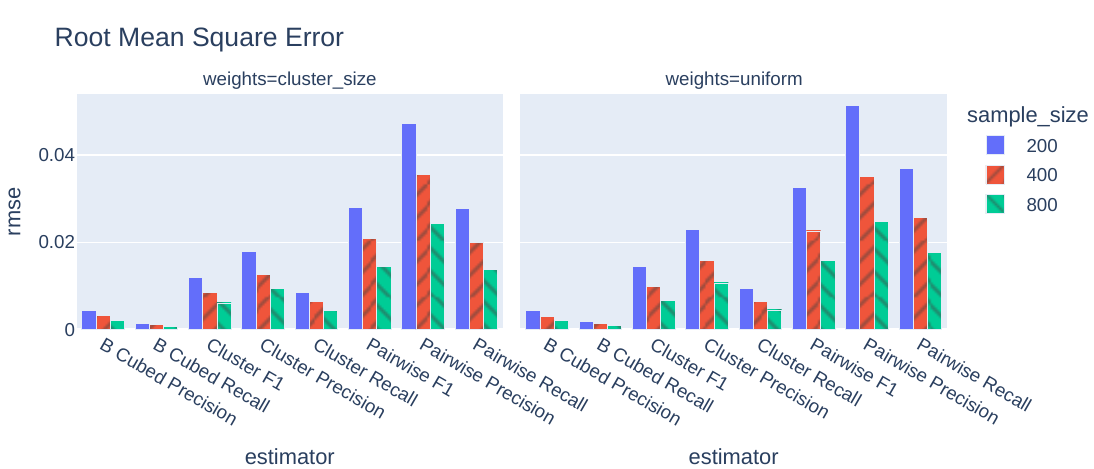}
    \includegraphics[width=0.85\linewidth]{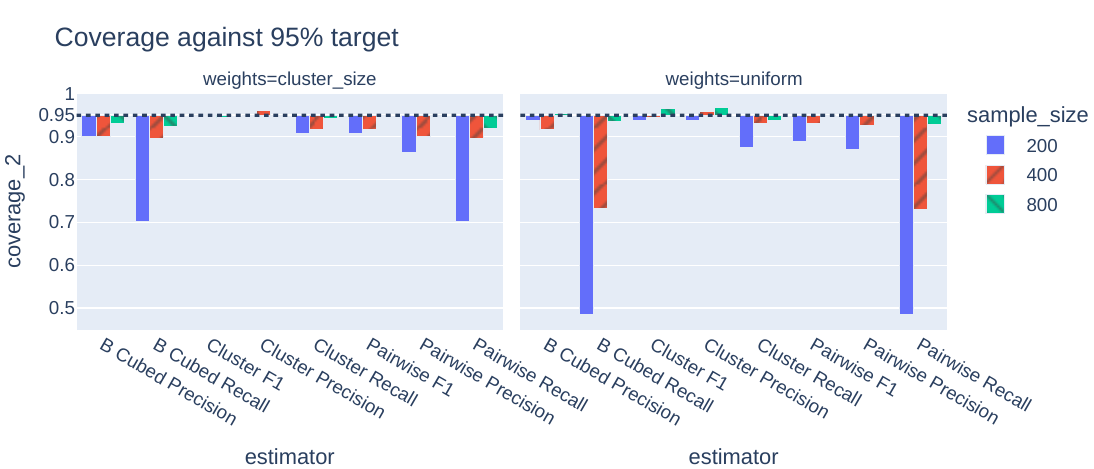}
    \caption{Simulation study based on the RLdata10000 data set. The accuracy of an ``all-but-one" matching algorithm was estimated by sampling ground truth clusters in a simulation replicated 1000 times per set of parameters. The estimates were compared to the known accuracy of the algorithm, specifically, $91\%$ pairwise precision and $97\%$ pairwise recall. Note that bias is below $0.4\%$ in all cases.}
    \label{fig:rldata-simulation}
\end{figure}

The RLData10000 data set \citep{RecordLinkage} is a synthetic data set containing 10,000 personal information records with first name, last name, birth date, birth month, and birth year. There are 1,000 clusters of size two and 8,000 singleton clusters. Individuals' names and birth dates sometimes appear with errors, and distinct individuals sometimes share the same names or birth dates.
As previously discussed, we consider an ``all-but-one'' matching algorithm for our predicted disambiguation. That is, we link together two records if and only if they match on four or five components among of first name, last name, birth year, birth month, and birth day. 

For evaluation, we sample clusters with replacement, either with probability proportional to cluster size (the default approach) or with uniform probabilities. These designs result in weights labeled respectively as  ``cluster\_size'' and ``uniform.'' We consider samples of sizes 200, 400, and 800. For each combination of parameters and estimator, we replicate the sampling and estimation process 1,000 times. Since we have ground truth for the RLData10000 data set, we can compute the bias and root mean square error (RMSE) of the point estimates. We also compute empirical coverage rates of approximate 95\%  confidence intervals. 
We consider coverage since the extent of any deviation from the nominal coverage is easy to visualize and understand. 
We note that it is generally 
challenging to achieve good coverage from large-sample confidence intervals,
especially when dealing with sparse data or skewed distributions, as the distribution of the estimator is only approximately normal. 

Figure \ref{fig:rldata-simulation} summarizes the results of 1,000 runs.  
The empirical bias is always smaller than $0.4\%$ and always less than $0.2\%$ when going up to samples of size at least 400. This validates the near unbiasedness of the estimators.

In terms of RMSE, the pairwise metric estimators are the least accurate, followed by cluster estimators, and then the highly accurate b-cubed estimators. To interpret the RMSE values for the pairwise precision estimator, we first point out that a data-free (and not recommended) estimator of precision equal to $100\%$ achieves RMSE of $9\%$, since the true pairwise precision is $91\%$. With probability proportional to cluster size sampling and at a sample of size 200, the RMSE is around $4.7\%$.  This decreases to $3.5\%$ at sample size 400 and $2.4\%$ at sample size 800. These RMSE values, while 
not insubstantial,
tend to be smaller than the corresponding RMSEs from uniform probability sampling. 
To see why the cluster and b-cubed estimators are more  accurate, note that  the  
pairwise estimators are defined in terms of pairs of records that are predicted to match or that are true matches. There are only $1,000$ matching pairs of records across the $9,000$ clusters in this data set, and there is a roughly similar number of predicted matching pairs. This makes estimation difficult as a large number of sampled clusters will not be associated with any predicted or matching pair. On the other hand, cluster and b-cubed metrics are defined relative to the populations of true and predicted clusters, for which we collect information in each sample.

Finally, we consider the coverage of the approximate 95\% confidence intervals. The ``coverage\_2'' label represents confidence intervals defined as the point estimate plus or minus two times the estimated standard deviation. 
The coverage rate for both precision and recall estimators is low at sample size 200. However, with sampling probability to cluster size, coverage rates are at least $90\%$ when using samples of size 400 or larger. At sample size 800, the coverage rate is roughly nominal. The coverage rates when sampling with uniform probability weights are lower in general.  This is due to the fact that errors are more rarely observed with a uniform design, leading to sparse data and a more variable standard deviation estimator. Overall, we attribute the less-than-nominal coverage rates in lower sampler samper size to two factors, namely non-normality of the estimator's sampling distribution and excessive variability in the standard deviation estimates.
This is evident in the 
distribution of the standard deviation estimator, displayed in Figure \ref{fig:std_distribution} for the pairwise precision and pairwise recall estimators when sampling with probability proportional to cluster size. With size $200$, the empirical distribution of the pairwise recall standard deviations has a point mass at $0$. This corresponds to cases where no underclustering errors were observed in the sample, leading to a recall estimate of $100\%$ with $0$ standard deviation. 

\begin{figure}[ht]
    \centering
    \includegraphics[width=\linewidth]{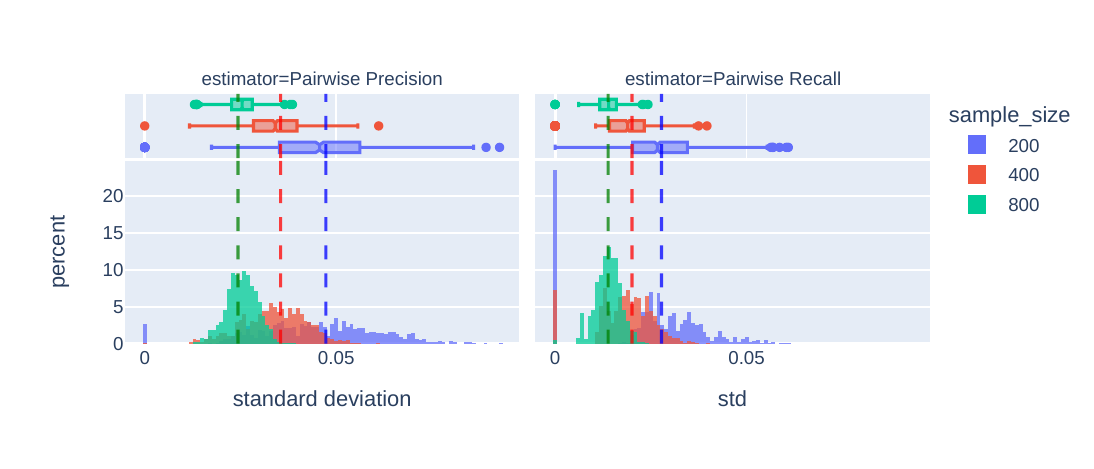}
    \caption{Distribution of the standard deviation estimator for pairwise precision and pairwise recall estimators, with probability proportional to cluster size sampling.}
    \label{fig:std_distribution}
\end{figure}

Overall, we  take away from the simulation study that the estimators can offer accurate reflections of the performance metrics, especially when using probability proportional to size sampling with sufficient sample sizes.  

\begin{figure}[ht]
    \centering
    \includegraphics[width=0.85\linewidth]{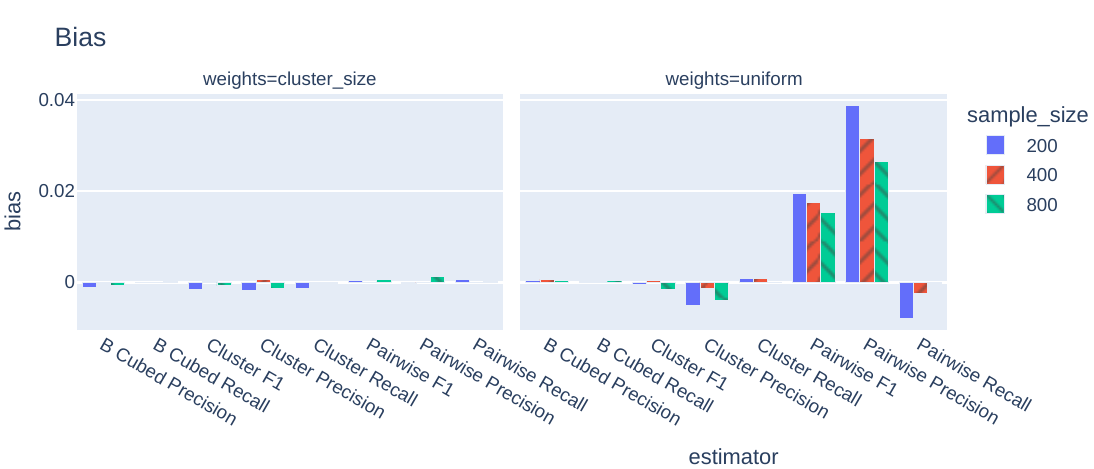}
    \includegraphics[width=0.85\linewidth]{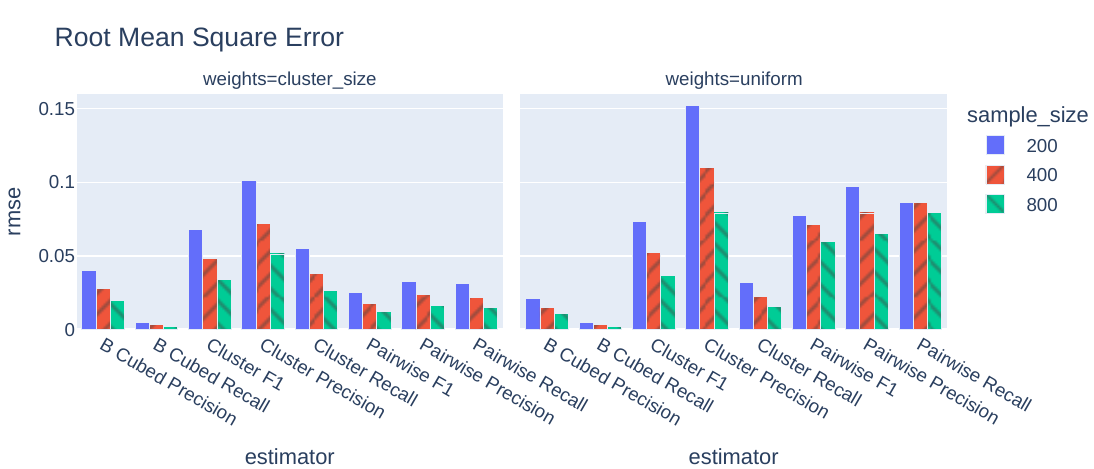}
    \includegraphics[width=0.85\linewidth]{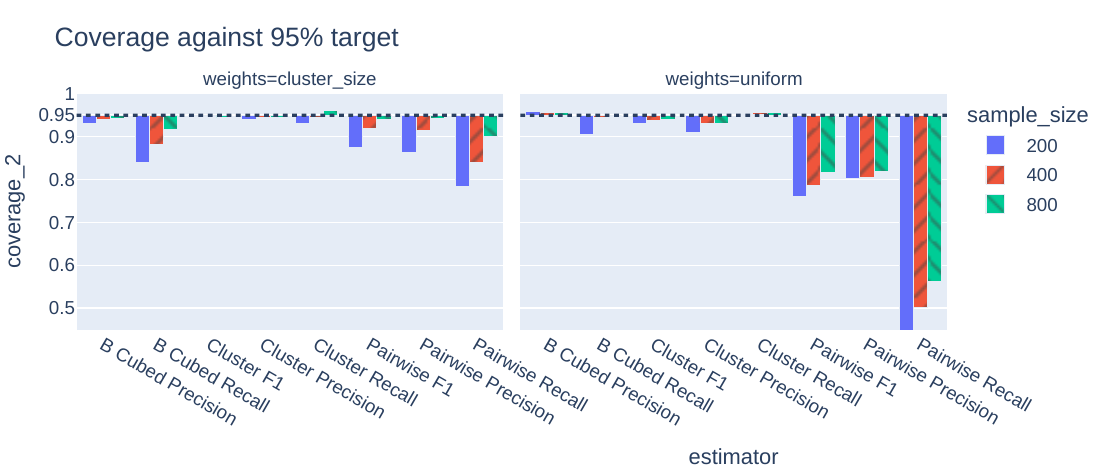}
    \caption{Simulation study based on the PatentsView's historical disambiguations, where the December 31, 2021 disambiguation was taken as a prediction, and the May 28, 2018 disambiguation was taken as ``ground truth". The accuracy of the predicted disambiguation was estimated by sampling ground truth clusters in a simulation replicated 1000 times per set of parameters. The estimates were compared to the known accuracy of the algorithm, specifically, $91\%$ pairwise precision and $94\%$ pairwise recall. Note that bias is essentially zero when sampling with probability proportional to cluster size (the default approach), and only substantial for pairwise precision estimates when sampling clusters with uniform probabilities.}
    \label{fig:pv-simulation}
\end{figure}

\subsubsection{PatentsView Data Simulation}

We now consider the simulation based on PatentsView data. 
We use the same parameters as for the RLData10000 simulation described in section \ref{sec:rldata-simulation}.  Figure \ref{fig:pv-simulation} displays the results of the simulation.

In terms of bias, uniform sampling weights are unreliable for estimating pairwise precision in this simulation. We do not recommend using uniform sampling weights for data like the PatentsView data, as large cluster sizes with many errors may not be observed frequently in the sample. On the other hand, the bias of estimators is negligible when using sampling with probability proportional to cluster sizes.

Considering RMSE, in this simulation, we see that  
sampling with probability proportional to cluster size leads to reasonable RMSEs for b-cubed and pairwise metrics estimators, and the uniform sampling design continues to be inadequate for PatentsView data, as evident by the large RMSE of the corresponding estimators. 
The reason for the large RMSE of cluster metrics estimates, when sampling with probability proportional to cluster size, is the same as for their 
unreliable confidence interval coverage discussed in \ref{sec:results_estimates}. That is, 
relatively few small clusters are sampled, despite small clusters being the most prevalent. Since the cluster metrics put equal weight on all clusters, with small clusters being the most prevalent, this leads to increased variability of the estimates.

Considering the coverage of confidence intervals, we see  similar behavior as in the RLData10000 simulation when sampling with probability proportional to cluster size. Coverage rates  approach the nominal 95\% level as the sample size grows. Sample sizes of 400 or 800 appear necessary to obtain reasonable coverage with these data. Confidence intervals of pairwise metrics are unreliable when sampling clusters uniformly.

Overall, the simulation study demonstrates the effectiveness our estimators when using simple random sampling of records, which is a convenient design for implementation and corresponds to sampling clusters with probability proportional to their size. Accurate estimates can be obtained for sample sizes that are practical in real applications.

\section{Discussion}\label{sec:discussion}

This paper introduces a novel evaluation framework for entity resolution systems. It ties an entity-centric data labeling methodology together with informative evaluation tasks, such as monitoring, performance estimation, and error analysis. Furthermore, the framework unifies many aspects of the evaluation process through the definition of two key metrics defined at the record or cluster level: the \textbf{overclustering error} and the \textbf{underclustering error} (section \ref{sec:error_space_definition}). All of the performance estimators are derived from these two metrics or simple variants, making it straightforward to relate error analysis with performance metric estimates and to extend the framework to estimate additional performance metrics.

We have demonstrated how our framework can be used in practice, without requiring the use of sophisticated sampling schemes. Once a weighted benchmark data set has been collected using our data labeling methodology, it can be used to evaluate multiple different disambiguation algorithms in multiple different ways. The labeled data are not tied to a singular algorithm or evaluation objective.  Furthermore, we validated the estimators and data labeling methodology in simulation studies, providing evidence that sampling clusters with probability proportional to size (via sampling records uniformly at random) can facilitate  accurate estimation of key metrics in different situations.

There is opportunity for future work on the design of refined sampling schemes and estimators. A finite population sampling point of view would be useful to accomodate smaller data sets and data sets with very large clusters. Adaptive sampling schemes, or sampling schemes derived from model-based estimates, could improve efficiency. Finally, more sophisticated ratio estimators and variance estimators could be considered and compared, including model-based or model-assisted estimators that use covariates available at the record level. Our unified evaluation framework provides opportunity for such sampling schemes and estimators to be useful for a large range of evaluation objectives. Furthermore, our framework accomodates the propagation of labeling uncertainty into estimates. This is an important topic that can be explored in more depth. Finally, it would be useful to be able to extrapolate the performance of a given algorithm over time, as more records are collected, or when applied to larger datasets. This could help extrapolate performance from artificial benchmarks to real data, or help anticipate performance degradations. Our current evaluation framework can be used as a starting point for these problems.

\section*{Software}
Our evaluation framework is implemented in the ``ER-Evaluation'' Python package \citep{Binette2023joss} available at \url{https://github.com/OlivierBinette/er-evaluation/}.

\section*{Author Contributions}

Olivier Binette led the project, the methodological research and development, and the writing. Youngsoo Baek contributed to the development of summary statistic estimators. Siddharth Engineer provided context and user analytics for PatentsView, using OCR and large language models to classify hundreds of papers citing PatentsView. Christina Jones contributed to the development of methods and provided context and user analytics for PatentsView. Abel Dasylva contributed the section on industry standards and to the methodology. Jerome P. Reiter contributed to the development of the statistical methods and the writing.

\appendix
\section{Proofs}

\begin{proof}[Proof of Lemma \ref{lemma:pairwise_precision_recall}]
First, we express $\lvert \mathcal{P} \rvert$ as a sum over $c\in \mathcal{C}$ through
    \begin{align}
        \lvert \mathcal{P} \rvert 
        = \sum_{\hat c \in \hat{\mathcal{C}}} {\lvert \hat c \rvert \choose 2}
        = \sum_{r\in \mathcal{R}} \frac{1}{\lvert \hat c(r) \rvert} {\lvert \hat c(r) \rvert \choose 2}
        = \frac{1}{2}\sum_{c\in \mathcal{C}} \sum_{r \in c} \left( \lvert \hat c(r) \rvert - 1 \right).\label{eq:proof_num_links}
    \end{align}
    For a given cluster $c \in \mathcal{C}$ and $r \in c$, adding and substracting $\lvert c \rvert$, we find $\lvert \hat c(r) \rvert - 1 = \lvert c \rvert - 1 + \texttt{SDE}(r)$. Substituting this expression into \eqref{eq:proof_num_links}, we obtain
    \begin{equation}\label{eq:P_expression}
        \lvert \mathcal{P} \rvert = \frac{1}{2} \sum_{c\in\mathcal{C}}\sum_{r\in c} \left( \lvert c \rvert - 1 + \texttt{SDE}(r) \right) = \frac{1}{2} \sum_{c\in\mathcal{C}}\lvert c \rvert \left( \lvert c \rvert - 1 + \texttt{SDE}(c) \right).
    \end{equation}
Similarly, we express $\lvert \mathcal{T} \cap \mathcal{P} \rvert$ as
    \begin{equation}
        \lvert \mathcal{T} \cap \mathcal{P} \rvert = \sum_{r \in \mathcal{R}} \frac{1}{\lvert \hat c(r) \cap c(r) \rvert} {\lvert \hat c(r) \cap c(r) \rvert \choose 2} = \frac{1}{2} \sum_{r\in \mathcal{R}}(\lvert \hat c(r) \cap c(r) \rvert - 1).
    \end{equation}
Using the fact that for $r \in c$ we have $\lvert \hat c(r) \cap c(r) \rvert = \lvert c \rvert - \texttt{UCE}(r)$, and averaging over $c \in \mathcal{C}$, we find
\begin{equation}\label{eq:P_cap_T_expression}
     \lvert \mathcal{T} \cap \mathcal{P} \rvert = \frac{1}{2}\sum_{r \in \mathcal{R}} (\lvert c \rvert - \texttt{UCE}(r)) = \frac{1}{2}\sum_{c\in \mathcal{C}}\lvert c \rvert (\lvert c \rvert - \texttt{UCE}(c)).
\end{equation}
Finally, it follows from definition that 
\begin{equation}\label{eq:T_expression}
\lvert \mathcal{T} \rvert = \frac{1}{2}\sum_{c\in \mathcal{C}} \lvert c \rvert (\lvert c \rvert - 1).
\end{equation}

The lemma follows directly from our expressions for $\lvert \mathcal{P} \rvert$, $\lvert \mathcal{P} \cap \mathcal{T} \rvert$, and $\lvert \mathcal{T} \rvert$ after re-expressing them as expectations over a random cluster $c$ distributed with probabilities proportional to $p_c > 0$.
\end{proof}

\begin{proof}[Proof of Lemma \ref{lemma:f-score}]
    First write
    \begin{equation}
        F_{\beta} = \frac{(1+\beta^2) \lvert \mathcal{T} \cap \mathcal{P} \rvert}{\lvert \mathcal{P} \rvert + \beta^2 \lvert \mathcal{T} \rvert}.
    \end{equation}
    Substituting \eqref{eq:P_expression}, \eqref{eq:T_expression}, and \eqref{eq:P_cap_T_expression} in the above, we obtain
    \begin{align}
        F_{\beta} = \frac{(1+\beta^2) \sum_{c\in \mathcal{C}} \lvert c \rvert (\lvert c \rvert - \texttt{UCE}(c)) }{\sum_{c\in \mathcal{C}}  \lvert c \rvert (\lvert c \rvert - 1 + \texttt{SDE}(c) + \beta^2 (\lvert c \rvert - 1)) }
        =\frac{\sum_{c\in \mathcal{C}} \lvert c \rvert (\lvert c \rvert - \texttt{UCE}(c)) }{\sum_{c\in \mathcal{C}}  \lvert c \rvert \left(\lvert c \rvert - 1 + \tfrac{1}{1+\beta^2}\texttt{SDE}(c)\right)}.
    \end{align}
    The lemma follows after re-expressing the sums as expectations over a random cluster $c$ distributed with probabilities proportional to $p_c > 0$.
\end{proof}

\begin{proof}[Proof of Lemma \ref{lemma:cluster_metrics}]

Write $\tilde p_c = p_c /\sum_{c'\in \mathcal{C}}p_{c'}$ for the normalized probability mass function of the random cluster $c$, and let $w_c = \lvert \mathcal{C} \rvert^{-1}/\tilde p_c$ be the ratio of the constant probability mass function to $\tilde p_c$.  From the fact that $N = \sum_{c\in \mathcal{C}} \lvert c \rvert$ we can derive 

\begin{equation}\label{eq:C_expr}
    \lvert \mathcal{C} \rvert = N/\mathbb{E}[\lvert c \rvert w_c].
\end{equation}

As such, we find
\begin{equation}\label{eq:C_cap_Chat_expr}
\lvert \mathcal{C} \cap \hat{\mathcal{C}}\rvert = \sum_{c\in\mathcal{C}} \texttt{EI}(c) = \lvert \mathcal{C} \rvert\, \mathbb{E}[\texttt{EI}(c)w_c] = \frac{N \mathbb{E}[\texttt{EI}(c)w_c]}{\mathbb{E}[\lvert c \rvert w_c]}
\end{equation}
and, after simplifying normalizing contants,
\begin{equation}
cP = \frac{\lvert \mathcal{C} \cap \hat{\mathcal{C}}\rvert}{\lvert \hat{\mathcal{C}} \rvert} = \frac{N \mathbb{E}[\texttt{EI}(c)/ p_c]}{\lvert \hat{\mathcal{C}} \rvert \mathbb{E}[\lvert c \rvert / p_c]}.
\end{equation}
The expression for $cR$ is a standard self-normalized importance sampling representation, i.e., 
\begin{equation}
cR = \frac{\lvert \mathcal{C} \cap \hat{\mathcal{C}}\rvert}{\lvert {\mathcal{C}} \rvert}= \sum_{c\in\mathcal{C}} \lvert \mathcal{C} \rvert^{-1} \texttt{EI}(c)
 = \frac{\mathbb{E}[\texttt{EI}(c)w_c]}{\mathbb{E}[w_c]} =\frac{\mathbb{E}[\texttt{EI}(c)/p_c]}{\mathbb{E}[1/p_c]}.
\end{equation}
Finally, using \eqref{eq:C_expr} and \eqref{eq:C_cap_Chat_expr}, we find
\begin{equation}
cF_\beta =\frac{(1+\beta^2) \lvert \mathcal{C} \cap \hat{\mathcal{C}}\rvert}{\lvert \hat{\mathcal{C}}\rvert + \beta^2\lvert\mathcal{C}\rvert} = \frac{(1+\beta^2) \mathbb{E}[\texttt{EI}(c)w_c]}{\lvert \hat{\mathcal{C}}\rvert/\lvert\mathcal{C}\rvert + \beta^2}=\frac{(1+\beta^2) \mathbb{E}[\texttt{EI}(c)/p_c]}{\lvert \hat{\mathcal{C}}\rvert \mathbb{E}[\lvert c \rvert / p_c]/N + \mathbb{E}[\beta^2/p_c]}.
\end{equation}

\end{proof}

\begin{proof}[Proof of Lemma \ref{lemma:b_cubed_metrics}]
    This lemma follows directly from definitions when using self-normalized importance sampling representations as above.
\end{proof}

\begin{proof}[Proof of Lemma \ref{lemma:lemma_5}]
    This follows directly from the definitions \eqref{eq:lemma_5_1}, \eqref{eq:lemma_5_2}, and \eqref{eq:lemma_5_3}.
\end{proof}

\bibliographystyle{chicago}
\bibliography{main}

\end{document}